\documentclass[twoside,11pt]{article}
\usepackage{jmlr2e}
\pdfoutput=1

\usepackage{amsmath,amssymb,mathtools,graphicx,url,array}

\usepackage[usenames,dvipsnames]{color}

\usepackage{epstopdf}

\usepackage{pdflscape,afterpage}

\def\R{\mathbb{R}}
\def\B{\mathcal{B}}
\def\N{\mathbb{N}}

\def\vol{\mathrm{vol}}

\def\tr{\mathrm{tr}}

\def\tW{\tilde \W}

\def\x{\mathbf{x}}
\def\y{\mathbf{y}}
\def\v{\mathbf{v}}
\def\w{\mathbf{w}}
\def\u{\mathbf{u}}

\def\X{\mathcal{X}}

\def\L{\mathcal{L}}
\def\Ll{\mathcal{L}(\lambda)}

\def\llk{\ell(\lambda,k)}
\def\l{\ell}
\def\C{\mathcal{C}}
\def\G{\mathcal{G}}
\def\U{\mathbf{U}}

\def\V{\mathbf{V}}
\def\W{\mathbf{W}}
\def\I{\mathbf{I}}
\def\D{\mathbf{D}}
\def\A{\mathbf{A}}

\def\S3VM{S$^3$VM}

\def\Ln{\mathbf{L}_{\mathrm{N}}}
\def\Lno{\mathbf{L}_{\mathrm{N}_0}}
\def\L{\mathbf{L}}

\title{Connecting Spectral Clustering to Maximum Margins and Level Sets} \date{}

\begin{document}


\begin{center}
{\huge \bf Connecting Spectral Clustering to Maximum Margins and Level Sets\\
}
\vspace{20pt}
{\large \bf David P. Hofmeyr}
\end{center}

\author{David P.\ Hofmeyr  dhofmeyr@sun.ac.za \\
 Department of Statistics and Actuarial Science\\
Stellenbosch University\\
Stellenbosch, South Africa
}



\begin{abstract}%
We study the connections between spectral clustering and the problems of maximum margin clustering, and estimation of the components of level sets of a density function. Specifically, we obtain bounds on the eigenvectors of graph Laplacian matrices in terms of the between cluster separation, and within cluster connectivity. These bounds ensure that the spectral clustering solution converges to the maximum margin clustering solution as the scaling parameter is reduced towards zero. %
The sensitivity of maximum margin clustering solutions to outlying points is well known, but can be mitigated by first removing such outliers, and applying maximum margin clustering to the remaining points. If outliers are identified using an estimate of the underlying probability density, then the remaining points may be seen as an estimate of a level set of this density function. We show that such an approach can be used to consistently estimate the components of the level sets of a density function under very mild assumptions.
\end{abstract}

\vspace{10pt}
%

\section{Introduction}

In maximum margin clustering, the objective is to obtain cluster separators for which the distance to the nearest data points is maximised. If no constraints are placed on the formulation of the cluster separators, then the maximum margin solution partitions data so that the between cluster distance is maximised. Such solutions are intuitively attractive, since we naturally associate similarities between data with how close they are in some metric space, most frequently Euclidean space. Maximising the between cluster Euclidean separation therefore seems like a sensible approach. However, such solutions are extremely sensitive to noise, and in many cases the clustering solution which maximises between cluster distance will only separate isolated points arising in the outer regions of a collection of data.

In the statistical approach to clustering, we imagine that our data arise from some probability distribution, and it is convenient to assume that this distribution comprises a mixture of simple components, each one of which representing a cluster. The most popular parametric model in this approach is the Gaussian mixture model (GMM). In this case, the maximum margin clustering solution will separate points in the tails of the mixture with extremely high probability as the sample size increases. This is because the density between mixture components is higher than it is in the tails. In fact, unless the clusters (mixture components) are supported on  disjoint, compact sets, it is generally very unlikely that unconstrained maximum margin solutions will be relevant for clustering.
A simple but effective approach to mitigating the effect of noise, or isolated tail observations, is to manually remove points which are believed to be in the tails of the underlying distribution, and only apply a large margin clustering method to the remaining points. If these tail points are identified using an empirical estimate of the underlying density, then we arrive at the very well known problem of level set estimation. 
%
%

Consider a probability density function, $p:\R^d \to \R^+$. Then the level set of $p$, at level $\lambda$, which we denote $\Ll$, is given by,
\begin{align}
\Ll = \{\x \in \R^d | p(\x) \geq \lambda\}.
\end{align}

\noindent
If $p$ is multimodal, then as $\lambda$ increases, the level set splits into multiple connected {\em components} which surround the modes of $p$. Each such component may then be associated with a cluster. This cluster definition has been widely adopted since the introduction of this formulation given by~\cite{hartigan1975clustering}. Implicitly then, clusters are associated with high density regions around each of the modes of the probability density. This is consistent with the intuition underlying the mixture model formulation, assuming the mixture components are prominent enough that they result in modes in the density. However, the level set formulation is not constrained by any parametric assumptions which arise in the explicit mixture model approaches, such as GMMs. It also allows the clusters to take on arbitrary shapes, where most practically adopted parametric mixture models result in convex, or nearly convex clusters. Notice also that that the truncation of the random variable $X$, with density $p$, within $\Ll$ may be seen as having a mixture density whose mixture components are the truncations of $X$ within the different components of $\Ll$. For $\lambda > 0$, except for pathological cases, these mixture components are supported on disjoint compact sets, and so any method which performs maximum margin clustering may be reasonably expected to be able to estimate the different components of $\Ll$.
%
%
One of the theoretical benefits of the level set approach to clustering is that, provided simple assumptions on the density function, $p$, it leads to a well posed statistical estimation problem. Indeed, numerous consistent procedures for the estimation of level set components have been proposed~\citep{walther1997granulometric,cuevas2000estimating,rinaldo2010generalized,pelletier2011operator}. 

In this paper we study the consistency of estimating level set components, using spectral clustering applied to a truncated sample based on an empirical estimate of $p$. 
Spectral clustering is a relatively recent approach to clustering which has become extremely popular for its flexibility and its comparative algorithmic simplicity. Spectral clustering obtains a relaxed solution of the normalised graph cut problem via the eigenvectors of the corresponding graph Laplacian matrix.
We study specifically the spectral clustering solutions for similarity graphs of points in Euclidean $\R^d$. We begin our analysis by deriving bounds on the eigenvectors of the Laplacian matrices. These bounds are used to show that the maximum margin clustering solution arises trivially from the spectral clustering solution, as the scaling parameter is reduced towards zero. We go on to obtain sufficient conditions on the convergence rate of the scaling parameter to consistently estimate $\Ll$, and ensure that, almost surely as $n\to\infty$, the components correspond with the maximum margin clustering solution. It is found that these rates are also sufficient to ensure that the spectral clustering solution recovers the maximum margin solution, and thus the components of $\Ll$. So far this assumes the number of components of $\Ll$ is known. We therefore also derive bounds on the eigenvalues of the same graph Laplacians, which allow us to consistently estimate the number of components of $\Ll$.

The remainder of the paper is organised as follows. In Section~\ref{sec:related} we discuss related work, and how our results extend on this body of literature. In Section~\ref{sec:spectral} we briefly discuss spectral clustering, and the formulation of graph Laplacian matrices from points in $\R^d$. The main results of the paper are given in Section~\ref{sec:main}. Sections~\ref{sec:evec_bounds} and~\ref{sec:eval_bounds} present derivations of bounds on the eigenvectors and eigenvalues of graph Laplacians respectively. In Section~\ref{sec:mmc} these bounds are shown to result in the convergence of spectral clustering to the maximum margin clustering solution, as the scaling parameter is reduced towards zero. Then in Section~\ref{sec:levsets} these results are placed in the context of level sets, and the consistency of the estimation procedure of applying spectral clustering to the truncated sample is shown. Finally we conclude with a discussion of the results in Section~\ref{sec:conclusion}.

\section{Relation to Existing Work}\label{sec:related}

Although multiple existing approaches for estimating components of level sets use graph theoretic partitioning algorithms, the only existing approach of which we are aware which is based directly on spectral clustering is that of~\cite{pelletier2011operator}. There the authors use an approach similar to that of~\cite{rinaldo2010generalized}, where first a consistent estimator of the underlying density is used to 
identify points in the estimated level set, and then a fixed bandwidth kernel is applied to these points to estimate the components/clusters. The authors show that under relatively mild assumptions their normalised spectral clustering algorithm consistently estimates the level set components. An important difference between this and our approach is that this existing work uses a fixed bandwidth parameter when applying spectral clustering. As a result, it is necessary that the distance between the components of the level set is known. Otherwise it is possible that this procedure will merge components which are close together. We consider the more natural case where the scaling parameter is reduced as the number of observations increases. In fact we find that the rate of convergence of the sequence of scaling parameters required for consistent estimation of the level set components using spectral clustering, is also sufficient for uniformly consistent estimation of the underlying density, and hence the level set itself. The same kernel computations used for estimating the level set are therefore also used in the spectral clustering step.
The approach of~\cite{pelletier2011operator} also applies only to kernels with bounded support. This makes the analysis simpler since, provided the bandwidth is smaller than half the distance between the level set components, the similarity graph of the points in the estimated level set is disconnected with high probability, and it is well known that spectral clustering recovers the components of a disconnected graph~\citep{Luxburg2007}. We extend this to allow kernels with unbounded support, provided the tails do not decay too slowly, and hence include the ubiquitous Gaussian kernel. Finally, our consistency analysis extends that of~\cite{pelletier2011operator} by considering the Laplacian matrices derived from the Ratio Cut as well as the Normalised Cut objective.

Arguably the most important existing work on the consistency of spectral clustering is the foundational work of~\cite{LuxburgConsistency}. There the authors investigate the consistency of spectral clustering in a general sense, rather than in relation to the estimation of a particular feature of the underlying distribution. In fact these authors also apply a fixed bandwidth kernel, and hence any asymptotic properties of the spectral clustering solution will be in relation to the convolution of the underlying distribution with the distribution whose density is given by the fixed bandwidth kernel. Other existing works which connect spectral clustering to the properties of the underlying distribution do so by studying the properties of the exact normalised cut solutions, and not the spectral clustering relaxations~\citep{narayanan2006relation,trillos2015consistency,hofmeyr2017improving}. These approaches are therefore fundamentally different from the present work.

Finally, as far as we are aware, the only existing work which connects spectral clustering to maximum margin clustering, is that of~\cite{hofmeyr2018minimum}. There the authors show that the optimal one-dimensional projection of a dataset for spectral clustering converges to the normal vector to the maximum margin hyperplane for clustering. The results in this existing work effectively ensure that the spectral clustering solution for points in $\R$ converges to the maximum margin solution. The large margin results presented here therefore extend these existing results to the multivariate setting.


\section{Graph Cuts and Spectral Clustering}\label{sec:spectral}

In this section we give a brief but explicit introduction to spectral clustering. For a very accessible and extended discussion on the topic, the reader is referred to~\cite{Luxburg2007}. Consider a collection of points, $\X = \{\x_1, ..., \x_n\}$, in $\R^d$. Now let $\G = (\X, \mathcal{E})$ be the graph with vertices given by the elements in $\X$, and where edge weights are determined by the similarities between pairs of points/vertices. That is, $\mathcal{E}_{i,j} = $ similarity$(\x_i,\x_j)$. It is common, and intuitively appealing, to determine similarities between points based on how close they are with respect to a metric, $d(\cdot, \cdot)$, on $\R^d$. That is, if $K$ is a non-increasing function on the non-negative real numbers, then it is common to define similarity$(\x_i, \x_j) = K(d(\x_i, \x_j))$. In this way, pairs of points which are nearer in space are assigned higher similarity than pairs which are more distant.

Now, a {\em cut} of a graph refers to a partition of its vertices through the removal of a subset of its edges.
 There is therefore an obvious bijection between the partitions/clusterings of $\X$ and the cuts of $\G$. We can therefore use the properties of graph cuts, and the optimisation problems associated with finding optimal cuts, to study the corresponding clustering solutions. Two popular graph cut objectives considered extensively in the clustering context are the {\em Ratio Cut} (RCut) and {\em Normalised Cut} (NCut). Stated explicitly in relation to the data set $\X$,
\begin{align*}
\mathrm{RCut}(\X) = \min_{\{\C_1, ..., \C_k\} \in \Pi_k(\X)} &\sum_{l=1}^k \frac{\mathrm{Cut}(\C_i, \X\setminus \C_i)}{|\C_i|},\\
\mathrm{NCut}(\X) = \min_{\{\C_1, ..., \C_k\} \in \Pi_k(\X)} &\sum_{l=1}^k \frac{\mathrm{Cut}(\C_i, \X\setminus \C_i)}{\vol(\C_i)},
\end{align*}
where we have used the notation $\Pi_k(\X)$ to refer to the collection of all $k$-way partitions of $\X$, and
\begin{align*}
\mathrm{Cut}(\C, \X\setminus \C) = \sum_{\substack{i,j:
\x_i\in \C,\\ \x_j \not \in \C}} K(d(\x_i,\x_j)), & \ \vol(\C) = \sum_{\substack{i,j:\x_i\in \C\\\x_j\in \X}} K(d(\x_i, \x_j)).
\end{align*}
Broadly speaking, solutions which minimise either RCut or NCut tend to correspond with solutions in which the total similarity between points in different clusters is low, but solutions containing very small clusters or clusters with low internal similarity are avoided through normalisation by either the cardinatity $|\cdot|$, or volume $\vol(\cdot)$, of the individual clusters.
Both RCut and NCut are attractive objectives for clustering, but obtaining the globally optimal solutions is NP-hard~\citep{Wagner1993}. Furthermore, obtaining high quality locally optimal solutions is not straightforward. Instead a relaxation is considered, in which the data are transformed using the eigenvectors of so-called {\em graph Laplacian} matrices. Clustering using the spectral decomposition of graph Laplacian matrices is referred to as spectral clustering.

Let $\A \in \R^{n\times n}$ be the {\em affinity} matrix for the graph $\G$, i.e., $\A_{i,j} = K(d(\x_i, \x_j))$, and $\D\in \R^{n\times n}$ the {\em degree matrix}, which is the diagonal matrix with $i$-th diagonal given by the sum of the $i$-th row of $\A$. Then the {\em unnormalised Laplacian} and {\em normalised Laplacian} of $\G$ are given respectively by $\L = \D-\A$ and $\Ln = \D^{-1/2}\L\D^{-1/2} = \I-\D^{-1/2}\A\D^{-1/2}$. For completeness we will consider two normalised Laplacians, where the second, which we denote $\Lno$, arises from the graph which is the same as $\G$ but the reflexive edges, i.e. those connecting vertices to themselves, are removed. Algebraically we have $\Lno = \I - \D_0^{-1/2}\A_0\D_0^{-1/2}$, where $\A_0$ is the same as $\A$ above, but with zeroes on its diagonal, and $\D_0$ has as diagonal the row sums of $\A_0$. As it turns out, in the context we consider, the differences between analysing $\Ln$ and $\Lno$ are far greater than those between $\L$ and $\Ln$. This arises from the fact that the diagonal elements of $\D_0$, unlike those of $\D$, are not bounded away from zero.

Now, it has been shown that the solution to the optimisation problem,
\begin{align}\label{eq:spec_unnorm}
\min_{\U\in\R^{n\times k}} \tr(\U^\top \L \U), \ \mbox{ such that } \U^\top \U = \I,
\end{align}
can be seen as a continuous relaxation of the cluster indicator vectors for the optimal RCut solution, scaled so that they form an orthonormal system~\citep{Hagen1992}. The solution to~(\ref{eq:spec_unnorm}) is given by the eigenvectors associated with the smallest $k$ eigenvalues of $\L$.
Similarly, the solution to the problem
\begin{align}
\min_{\U\in \R^{n\times k}} \tr(\U^\top \L\U), \mbox{ such that } \U^\top \D^{-1}\U = \I,
\end{align}
has as columns relaxations of scaled cluster indicator vectors for the optimal NCut solution~\citep{Shi2000}. In this case the solution can be shown to be given by $\D^{-1/2}\U$, where the columns of $\U$ are the first $k$ eigenvectors of $\Ln$.

In the remainder we study the properties of the graph Laplacian matrices in terms of their eigenvectors and eigenvalues, and show how these can be used to obtain maximum margin clustering solutions and to consistently estimate the components of the level sets of a density function on $\R^d$. Specifically, we will show that the matrix whose columns are the first $k$ eigenvectors converges to one which trivially exposes the maximum margin clustering solution, as the similarities become more and more locally concentrated. We use the same supporting results to show further that by applying spectral clustering to truncations of an increasing sample from a continuous probability distribution on $\R^d$, we can consistently estimate the components of the level sets of its density. The eigenvalues of the corresponding graph Laplacians are used to consistently estimate the number of components, while the eigenvectors are shown to trivially recover the partition of the points in the level set into its different components. 

%


\section{Connecting Spectral Clustering to Maximum Margins and Level Sets}\label{sec:main}

In this section we present complete derivations of the theoretical contributions of this paper. We first derive bounds on the eigenvectors and eigenvalues of graph Laplacian matrices, in terms of the within cluster connectedness and between cluster separation. We go on to show that, given mild assumptions on the similarity function, that as the scaling parameter is reduced to zero the spectral clustering solution converges to the maximum margin clustering solution, in the sense that within cluster distances (within the eigenvector representation) converge to zero, while between cluster distances are bounded below. For both the unnormalised Laplacian, $\L$, and the normalised Laplacian, $\Ln$, these bounds arise fairly straightforwardly. However, in the case of the normalised Laplacian, $\Lno$, derived from the graph with refelxive edges removed, no such lower bound on the between cluster distances is immediately forthcoming. Instead, in this case, we show that within cluster distances converge to zero at a much faster rate than between cluster distances, therefore having the same practical relevance of exposing the maximum margin clustering solution clearly. Finally we go on to establish conditions on the rate of convergence of the scaling paramater, in the context of an increasing sample arising from a continuous probability distribution on $\R^d$, in order to simultaneously and consistently estimate the level set; the number of components of the level set; as well as ensure that spectral clustering recovers the partition of points in the level set according to the components in which they lie.

We begin this section by introducing notation and terminology which will be important in the remaining paper. We then also introduce, and briefly discuss, the assumptions used in the theoretical analysis which follows.

\subsubsection{Notation and Terminology}

%
Most of the notation used in the remainder is fairly standard, but for completeness we list what is not universally employed, as well as terminology which we introduce for convenience in the following discussions.
 
For natural number $n\in \N$, we will write $[n]$ for the set containing the first $n$ natural numbers, i.e., $[n] = \{1, 2, ..., n\}$.
 Any use of the generic norm notation, $||\cdot||$, will refer to the Euclidean, or $L_2$ norm. Similarly any reference to a metric, $d(\cdot, \cdot)$, will thus correspond to the Euclidean metric, i.e., $d(\x, \y) = ||\x-\y||$ for any $\x, \y\in\R^d$.
 %
For $S, U\subset \R^d$ and $\x \in \R^d$ we will use $d(\x, S) = \inf_{\y\in S}d(\x, \y)$ to denote the distance between $\x$ and the set $S$, and $d(S, U) = \inf_{\x\in S, \y\in U}d(\x, \y)$ to denote the distance between the sets $S$ and $U$. By default we set $d(\x, \emptyset) = \infty$ for all $\x \in \R^d$, where $\emptyset$ is the empty set.
 We will use $\B_\epsilon(\x):= \{\y \in \R^d| d(\x, \y) < \epsilon\}$ to denote the $\epsilon$-neighbourhood of $\x\in\R^d$, and we will also write $\B_\epsilon(S):=\bigcup_{\x\in S} \B_\epsilon(\x)$ for the $\epsilon$-neighbourhood of a set $S\subset\R^d$.
 We say that a set $S\subset\R^d$ is connected at distance $\delta$ if there is no binary partition of $S$ into $S_1, S_2$ such that $d(S_1, S_2)>\delta$. Equivalently, $S$ is connected at distance $\delta$ if the closure of $\B_{\delta/2}(S)$ is a connected set. 
 The similarity function, $K:\R^+\to\R^+$, will be referred to as a {\em kernel}, and for $\sigma > 0$ we will use $K_\sigma$  to denote the scaled kernel, $K_\sigma(x):=K(x/\sigma)$.
 We will use $\G = (\X, K_\sigma)$ to denote the graph with vertices $\X$ and edges given by the similarities obtained by applying $K_\sigma$ to the distances between points in $\X$.
Finally, if $\mathbf{U}\in \R^{n\times m}$ is a matrix then we will write $\mathbf{U}_{a:b,c:d}$ for the sub-matrix containing rows $a, a+1, ..., b-1, b$ and columns $c, c+1 ..., d-1, d$. We will just write~$:$ for all rows/columns, and use just a single index as is usual for a single row/column. For example the matrix $\U_{1:5,:}$ contains the first five rows and all columns of $\U$.

\subsubsection{Assumptions on the kernel function, $K$}

As mentioned previously, we present our analysis for kernels with unbounded support. It is worth noting that the results will hold for kernels with bounded support, after only minor changes to the presentation herein. Once again, in the case where the support of the kernels is bounded, the similarity graph becomes disconnected as the scaling parameter is reduced, and hence the recovery of the solution by spectral clustering is immediate~\citep{Luxburg2007}. It is therefore the unbounded support case which we find far more interesting. In particular, we present results for kernels satisfying the following,
\begin{description}
\item[AK1:] $K$ is non-increasing and strictly positive on $[0, \infty)$.
\item[AK2:] $K(0) = 1$, $c_K\int_{\R^d}K(||\x||)d\x =  1$.
\item[AK3:] $\exists A, \alpha > 0$ such that $K(x)/K(y) \leq A\exp(-(x-y)^\alpha)$ for all $0\leq y \leq x$.
\end{description}
 Assumption AK1 is very standard, and intuitively desireable for determining similarity, since it ensures that pairs which are closer are assigned higher similarity than pairs which are further apart. Assumption AK2 can always be achieved by scaling all similarities, provided the integral $\int_{\R^d}K(||\x||)d\x$ is finite. The normalisation constant $c_K$ will be relevant when considering the estimation of the density using $K$. Assumption AK3 places an upper bound on the tail decay of the kernel, and excludes polynomially decaying tails, but includes, for example, the ubiquitous Gaussian kernel.

\subsubsection{Assumptions on the density, $p$, and level set $\Ll$}

We also make a few simplifying assumptions on the density $p$, and the level set of interest. It is certainly possible to relax these assumptions in favour of weaker ones, however we prefer to make assumptions which are stated as simply as possible. Furthermore, any distribution can be approximated arbitrarily well by one with a density which obeys the following,
\begin{description}
\item[A1:] We assume that $p$ has bounded first derivative, so that $||\nabla f(\x)||_2 < \kappa$ for all $\x\in \R^d$.
\item[A2:] We assume that $\exists C, \gamma > 0$ s.t. $\forall 0 < g \leq \gamma$ we have $$
\sup_{\x \in \mathcal{L}(\lambda-g)\setminus \mathcal{L}(\lambda)} d(\x, \Ll) \leq gC.
$$
\end{description}

\noindent
%
Assumption A1 allows us to use the uniform consistency of kernel density estimators, and also ensures there are finitely many components of the level set $\Ll$. 
Assumption A2 is a convenient way of stating that the density is not allowed to be too flat at levels at and just below $\lambda$.

\subsection{Eigenvector Bounds for Graph Laplacians}\label{sec:evec_bounds}

In this section we derive bounds on the distances between points in the same clusters, when mapped into the Laplacian eigenvector representation through spectral clustering. These bounds are expressed in terms of the within cluster connectedness, and the between cluster separation only, and so can be used directly to relate the spectral clustering solution to the maximum margin clustering solution. These results only place upper bounds on the pairwise distances between points from the same clusters, and do not directly ensure that points in different clusters are distinguishable. To achieve this we present general results which can be seen as providing lower bounds on the between cluster separation for any data set with full column rank, in terms of the within cluster distortion. We later combine these results to show that the spectral clustering solution converges to the maximum margin solution, as the scaling parameter is reduced towards zero.

The following three results respectively provide the upper bounds on the within cluster distances in the eigenvectors of the unnormalised Laplacian, $\L$, normalised Laplacian, $\Ln$, and normalised Laplacian from the graph with reflexive edges removed, $\Lno$.

\begin{lemma}\label{thm:evec_bound_1}
Let $\X = \{\x_1, ..., \x_n\}$, and let $\C_1, ..., \C_k$ be a partition of $\X$. For each $l\in [k]$, suppose $\C_l$ is connected at distance $\delta_l$. Let $\U \in \R^{n\times n}$ have as columns the eigenvectors of the unnormalised Laplacian of the graph $\G = (\X, K_\sigma)$. Then for each $i,j\in[n],l\in[k]$ s.t. $\x_i, \x_j \in \C_l$, we have
\begin{align*}
||\U_{i,1:k} - \U_{j,1:k}|| \leq \max_{m\in[k]} n^{1.5}k^{0.5}\sqrt{\frac{K_\sigma(d(\C_m, \X\setminus\C_m))}{K_\sigma(\delta_l)}}.
\end{align*}
\end{lemma}

\begin{proof}
Since spectral clustering is a relaxation of the Ratio Cut problem, we have
\begin{align*}
\sum_{i=1}^k \U_{:,i}^\top \L \U_{:,i} &\leq \min_{\{C_1, ..., C_k\}\in \Pi_k(\X)}\sum_{i=1}^k\sum_{j,l: \x_j \in C_i, \x_l\not \in C_i} \frac{K_\sigma(||\x_j - \x_l||)}{|C_i|}\\
&\leq \sum_{i=1}^k\sum_{j,l: \x_j \in \C_i, \x_l\not \in \C_i} \frac{K_\sigma(||\x_j - \x_l||)}{|\C_i|}\\
&\leq \sum_{i=1}^k\sum_{j,l: \x_j \in \C_i, \x_l\not \in \C_i} \frac{K_\sigma(d(\C_i, \X\setminus \C_i))}{|\C_i|}\\
&\leq \sum_{i=1}^k |\X\setminus \C_i|K_\sigma(d(\C_i, \X\setminus \C_i))\\
&\leq nk \max_{m\in[k]} K_\sigma(d(\C_m, \X\setminus \C_m))
\end{align*}
Now take any $l\in[k]$. Since $\C_l$ is connected at distance $\delta_l$, there exist $|\C_l|-1$ pairs of points in $\C_l$ with indices $(i_1, j_1), ..., (i_{|\C_l|-1}, j_{|\C_l|-1})$ s.t. $||\x_{i_m} - \x_{j_m}|| \leq \delta_l$ for each $m \in [|\C_l|-1]$ and the union of all such $\{\x_{i_m}, \x_{j_m}\}$ is equal to $\C_l$.

By~\cite[Proposition 1]{Luxburg2007} we know that for any $\u \in \R^n$ we have
\begin{align*}
\u^\top L\u = \frac{1}{2}\sum_{i,j}K_\sigma(||\x_i - \x_j||)(\u_i - \u_j)^2.
\end{align*}
We therefore have
\begin{align*}
\sum_{i=1}^k \U_{:,i}^\top \L \U_{:,i} &= \frac{1}{2}\sum_{i,j}K_\sigma(||\x_i - \x_j||) ||\U_{i,1:k} - \U_{j,1:k}||^2\\
\Rightarrow K_\sigma(||\x_{i_m} - \x_{j_m}||)||\U_{i_m,1:k}-\U_{j_m,1:k}||^2 &\leq \sum_{i=1}^k \U_{:,i}^\top \L \U_{:,i} \mbox{ for each } m \in[|\C_l|-1]\\
\Rightarrow K_\sigma(\delta_l)||\U_{i_m,1:k}-\U_{j_m,1:k}||^2 &\leq \sum_{i=1}^k \U_{:,i}^\top \L \U_{:,i} \mbox{ for each } m \in[|\C_l|-1]\\
\Rightarrow ||\U_{i_m,1:k}-\U_{j_m,1:k}|| &\leq \sqrt{\frac{\sum_{i=1}^k \U_{:,i}^\top \L \U_{:,i}}{K_\sigma(\delta_l)}} \mbox{ for each } m \in[|\C_l|-1]\\
\Rightarrow ||\U_{i,1:k}-\U_{j,1:k}|| &\leq |\C_l|\sqrt{\frac{\sum_{i=1}^k \U_i^\top \L \U_i}{K_\sigma(\delta_l)}} \mbox{ for any } i, j \mbox{ s.t. } \x_i, \x_j \in \C_l,
\end{align*}
where the final step comes from the triangle inequality, since all points in $\C_l$ are connected by the pairs $\x_{i_m}, \x_{j_m}, m \in[|\C_l|-1]$.
Putting these together, we have for any $\x_i, \x_j \in \C_l$, that
\begin{align*}
||\U_{i,1:k} - \U_{j,1:k}|| &\leq \max_{m\in[k]}|C_l|\sqrt{\frac{nk K_\sigma(d(\C_m, \X\setminus\C_m))}{K_\sigma(\delta_l)}}\\
&\leq \max_{m\in[k]}n^{1.5}k^{0.5}\sqrt{\frac{K_\sigma(d(\C_m, \X\setminus\C_m))}{K_\sigma(\delta_l)}},
\end{align*}
as required.
\end{proof}

What we obtain from the above is that if clusters are internally connected at smaller distances than the distances between clusters, then because of assumption AK3 we know that as $\sigma \to 0$, the ratio $K_\sigma(d(\C_m, \X\setminus \X_m))/K_\sigma(\delta_l)$ converges to zero. The result for the normalised Laplacian is extremely similar, with the main difference coming from the fact that the approximate normalised cut solution is given by $\D^{-1/2}\U$, and not the eigenvectors alone.

\begin{lemma}\label{thm:evec_bound_2}
Let $\X = \{\x_1, ..., \x_n\}$ and let $\C_1, ..., \C_k$ be a partition of $\X$. For each $l\in [k]$, suppose $\C_l$ is connected at distance $\delta_l$. Let $\U\in \R^{n\times n}$ have as columns the eigenvectors of the normalised Laplacian of the graph $\G = (\X, K_\sigma)$, and let $\mathbf{D}$ be the corresponding degree matrix. Then for each $i,j\in[n], l\in[k]$ s.t. $\x_i, \x_j \in \C_l$, we have
\begin{align*}
||\mathbf{D}_{ii}^{-1/2}\U_{i,1:k} - \mathbf{D}_{jj}^{-1/2}\U_{j,1:k}|| \leq \max_{m\in[k]} n^{1.5}k^{0.5}\sqrt{\frac{K_\sigma(d(\C_m, \X\setminus\C_m))}{K_\sigma(\delta_l)}}.
\end{align*}
\end{lemma}

\begin{proof}
The proof is very similar to before. The fact that since spectral clustering is a relaxation of the normalised graph cut problem now gives us,
\begin{align*}
\sum_{i=1}^k \U_{:,i}^\top \Ln \U_{:,i} &= \sum_{i=1}^k \U_{:,i}^\top \D^{-1/2}\L\D^{-1/2} \U_{:,i}
\leq \min_{\{C_1, ..., C_k\}\in\Pi_k(\X)}\sum_{i=1}^k\sum_{j,l: \x_j \in C_i, \x_l\not \in C_i} \frac{K_\sigma(||\x_j - \x_l||)}{\vol(C_i)}\\
&\leq \sum_{i=1}^k\sum_{j,l: \x_j \in \C_i, \x_l\not \in \C_i}  \frac{K_\sigma(||\x_j - \x_l||)}{\vol(\C_i)}
\leq \sum_{i=1}^k |\X\setminus \C_i| K_\sigma(d(\C_i, \X\setminus \C_i))\\
&\leq n k \max_{m\in[k]} K_\sigma(d(\C_m, \X\setminus\C_m)),
\end{align*}
where since $\D_{jj}\geq 1$ for each $j\in[n]$ we get $|\C_i|\leq \vol(\C_i)$ for all $i\in[k]$. As before we have for any $\u \in \R^n$ that
\begin{align*}
\u^\top \L\u &= \frac{1}{2}\sum_{i,j}K_\sigma(||\x_i - \x_j||)\left(\u_i - \u_j\right)^2\\
\Rightarrow 
\sum_{i=1}^k \U_{:,i}^\top \Ln \U_{:,i} =\sum_{i=1}^k \U_{:,i}^\top \D^{-1/2}\L\D^{-1/2} \U_{:,i} &= \frac{1}{2}\sum_{i,j}K_\sigma(||\x_i - \x_j||) \left\|\frac{\U_{i,1:k}}{\D_{ii}^{1/2}} - \frac{\U_{j,1:k}}{\D_{jj}^{1/2}}\right\|^2\\
\Rightarrow K_\sigma(||\x_{i_m} - \x_{j_m}||)\left\|\frac{\U_{i_m,1:k}}{\D_{i_mi_m}^{1/2}} - \frac{\U_{j_m,1:k}}{\D_{j_mj_m}^{1/2}}\right\|^2 &\leq \sum_{i=1}^k \U_i^\top\Ln \U_i \mbox{ for each } m \in [|\C_l|-1]\\
\Rightarrow \left\|\frac{\U_{i_m,1:k}}{\D_{i_mi_m}^{1/2}} - \frac{\U_{j_m,1:k}}{\D_{j_mj_m}^{1/2}}\right\|^2 &\leq \sqrt{\frac{\sum_{i=1}^k \U_i^\top\Ln \U_i}{K_\sigma(\delta_l)}} \mbox{ for each } m \in [|\C_l|-1]\\
\Rightarrow \left\|\frac{\U_{i,1:k}}{\D_{ii}^{1/2}} - \frac{\U_{j,1:k}}{\D_{jj}^{1/2}}\right\|^2 &\leq \sqrt{\frac{\sum_{i=1}^k \U_i^\top\Ln \U_i}{K_\sigma(\delta_l)}} \mbox{ for any } i, j \mbox{ s.t. } \x_i,\x_j \in \C_l,
\end{align*}
where we have used the same pairs $(\x_{i_1}, \x_{j_1}), ..., (\x_{i_{|\C_l|-1}}, \x_{j_{|C_l|-1}})$ as in the previous proof.
Putting these together as before gives the result.
\end{proof}

Crucial in the proof of the above result is the fact that the diagonals of $\D$ are bounded below by 1, since each point is linked to itself in the similarity graph. Without a fixed lower bound on the diagonal elements of $\D$, the bounds become weaker, as seen in the following Lemma. Additional requirements will be needed to ensure that spectral clustering recovers the maximum margin clustering solution in this case. These will be discussed explicitly in the relevant section to follow. 

\begin{lemma}\label{thm:evec_bound_3}
Let $\X = \{\x_1, ..., \x_n\}$ and let $\C_1, ..., \C_k$ be a partition of $\X$. For each $l\in[k]$, suppose that $\C_l$ is connected at distance $\delta_l$. Let $\U\in\R^{n\times n}$ have as columns the eigenvectors of the normalised Laplacian of the graph $\G = (\X, K_\sigma)$ but with reflexive edges removed, and let $\mathbf{D}$ be the corresponding degree matrix. Then for each $i,j\in[n],l\in[k]$ s.t. $\x_i, \x_j \in \C_l$, we have
\begin{align*}
||\mathbf{D}_{ii}^{-1/2}\U_{i,1:k} - \mathbf{D}_{jj}^{-1/2}\U_{j,1:k}|| \leq \max_{m\in[k]} n^{1.5}k^{0.5}\sqrt{\frac{K_\sigma(d(\C_m, \X\setminus\C_m))}{K_\sigma(\delta_m)K_\sigma(\delta_l)}}.
\end{align*}
\end{lemma}

\begin{proof}
The proof is exactly as in the previous lemma, except that now we have $\D_{jj} \geq K_\sigma(\delta_l)$ for all $j\in C_l$, and hence $\vol(\C_l) \geq |\C_l|K_\sigma(\delta_l)$ instead of $\vol(\C_l) \geq |\C_l|$, and hence
\begin{align*}
\sum_{i=1}^k \U_{:,i}^\top \Ln \U_{:,i} &\leq nk\max_{m\in[k]} \frac{K_\sigma(d(\C_m,\X\setminus\C_m))}{K_\sigma(\delta_m)}.
\end{align*}
rather than
\begin{align*}
\sum_{i=1}^k \U_{:,i}^\top \Ln \U_{:,i} &\leq nk\max_{m\in[k]} K_\sigma(d(\C_m,\X\setminus\C_m)).
\end{align*}
\end{proof}

\begin{remark}
We note that some authors recommend placing a lower bound on the diagonals of the degree matrix to enhance the stability of the eigenvector solver being used. From a practical point of view, therefore, allowing the elements of $\D$ to approach zero may be undesirable. Note that any fixed lower bound would ensure convergence of the spectral clustering solution to the maximum margin clustering. It is still interesting, however, to investigate theoretically the requirements needed in the event that no such lower bound is in place.
\end{remark}

The above results place upper bounds on the within cluster distances in the eigenvalue representation, in terms of the connectedness and separation of clusters in the input space. The following general results allow us to place lower bounds on the between cluster distances within the eigenvector representation. Although not explicitly related to eigenvectors, the following proposition may be easily placed in relation to the unnormalised Laplacian, since the eigenvectors used in clustering are orthogonal. On the other hand, in the normalised solution we use the matrix $\D^{-1/2}\U$, which does not have orthogonal columns. In the first corollary to the following result we provide a more general result which admits such matrices.

\begin{proposition}\label{thm:evec_sep_1}
Let $\V \in \R^{n\times k}$ have orthonormal columns. Suppose that $\W\in \R^{k\times k}$ satisfies,
\begin{align*}
\max_{i \in [n]}\left\{\min_{l\in[k]} ||\V_{i,:} - \W_{l,:}||\right\}=\epsilon.
\end{align*}
%
Then, provided $\epsilon < (3nk^2)^{-1}$, we have
$$
\min_{i,j\in [k], i\not = j} ||\W_{i,:} - \W_{j,:}|| \geq \sqrt{\frac{2}{n}} - \sqrt{12}k(3n\epsilon)^{1/4}.
$$
\end{proposition}

\begin{proof}
First, for each $i\in [n]$, let $c(i)\in[k]$ be such that $||\V_{i,:}-\W_{c(i),:}|| \leq \epsilon$. Then, for each $l\in[k]$, let $n(l) = \sum_{i=1}^n \mathbf{1}_{[c(i)=l]}$, where $\mathbf{1}_{[A]}$ is the indicator function for $A$. Note that we lose no generality by assuming that $n_l \geq 1$ for each $l\in[k]$, since $\V$ has rank $k$ and so contains at least $k$ unique rows. Then define $\tilde \W = \mbox{diag}(\sqrt{n_1}, ..., \sqrt{n_k})\W$. 
Now consider that, for any $i\in [k]$,
\begin{align*}
\sum_{l=1}^n\V_{l,i}^2 & = 1\\
\Rightarrow \sum_{l=1}^n (\V_{l,i}-\W_{c(l),i}+\W_{c(l),i})^2 &= 1\\
\Rightarrow \sum_{l=1}^n \W_{c(l),i}^2 + \sum_{l=1}^n (\V_{l,i}-\W_{c(l),i})^2 + 2\sum_{l=1}^n (\V_{l,i}-\W_{c(l),i})\W_{c(l),i} &= 1\\
\Rightarrow \bigg|||\tilde \W_{,:i}||^2-1\bigg| = \bigg|\sum_{l=1}^k n_l \W_{l,i}^2-1\bigg|= \bigg|\sum_{l=1}^n \W_{c(l),i}^2-1\bigg| \leq n\epsilon^2 + 2n\epsilon(1+\epsilon) &\leq 3n\epsilon,
\end{align*}
where we have used the fact that the elements of $\V$ are bounded between $-1$ and $1$, and hence the elements of $\W$ are bounded between $-(1+\epsilon)$ and $1+\epsilon$. Similar to above, for any $i,j\in[k]$ we have,
\begin{align*}
\sum_{l=1}^n \V_{l,i}\V_{l,j} &= 0\\
\Rightarrow \bigg|\tilde\W_{:,i}^\top \tilde \W_{:,j}\bigg| = \bigg| \sum_{l=1}^k n_l \W_{l,i}\W_{l,j} \bigg| = \bigg| \sum_{l=1}^n \W_{c(l),i}\W_{c(l),j} \bigg| &\leq 3n\epsilon.
\end{align*}
We therefore have $||\tilde \W^\top \tilde \W - \I||_{\infty} \leq 3n\epsilon$. Weyl's inequality~\citep{weyl1912} ensures that the eigenvalues of $\tW^\top \tW$ lie in $\left[1-k\sqrt{3n\epsilon}, 1+k\sqrt{3n\epsilon}\right]$. Note that $\epsilon < (3nk^2)^{-1} \Rightarrow k\sqrt{3n\epsilon} < 1$ and hence $\tW^\top\tW$ is non-singular.
So consider the matrix $\W^\star:= \tW(\tW^\top \tW)^{-1/2}$. It is simple to check that $\W^\star$ is orthogonal. Now let $||\cdot ||_F$ be the Frobenius norm, and consider
\begin{align*}
||\tilde \W - \tilde \W(\tilde \W ^\top \tilde \W)^{-1/2}||_F^2 &= ||\tilde \W( \I - (\tilde \W ^\top \tilde \W)^{-1/2})||_F^2\\
&= \mbox{tr}\left(\tilde \W( \I - (\tilde \W ^\top \tilde \W)^{-1/2})( \I - (\tilde \W ^\top \tilde \W)^{-1/2}) \tilde \W^\top\right)\\
&= \mbox{tr}\left(\tilde \W^\top\tilde \W( \I - 2(\tilde \W ^\top \tilde \W)^{-1/2} + (\tilde \W ^\top \tilde \W)^{-1}) \right)\\
&= \mbox{tr}\left(\tilde \W^\top\tilde \W( \I - 2(\tilde \W ^\top \tilde \W)^{-1/2} + (\tilde \W ^\top \tilde \W)^{-1}) \right)\\
&= \mbox{tr}\left(\tilde \W^\top\tilde \W\right) - 2\tr\left((\tilde \W^\top\tilde \W)^{1/2}\right) + k.
\end{align*}
We thus find
\begin{align*}
||\tW-\W^\star||_F^2 \leq 3k^2\sqrt{3n\epsilon},
\end{align*}
since the eigenvalues of $\tW^\top \tW$ lie in $\left[1-k\sqrt{3n\epsilon}, 1+k\sqrt{3n\epsilon}\right]$, and hence the eigenvalues of $(\tW^\top \tW)^{1/2}$ also lie in 
$\left[1-k\sqrt{3n\epsilon}, 1+k\sqrt{3n\epsilon}\right]$.
Finally, we have,
\begin{align*}
\left\|\frac{1}{\sqrt{n_i}}\W^\star_{i,:} - \frac{1}{\sqrt{n_j}}\W^\star_{j,:}\right\| &= \left\|\frac{1}{\sqrt{n_i}}\W^\star_{i,:} -\W_{i,:} + \W_{i,:} - \W_{j,:} + \W_{j,:} - \frac{1}{\sqrt{n_j}}\W^\star_{j,:}\right\|\\
&\leq \left\|\frac{1}{\sqrt{n_i}}\W^\star_{i,:} -\W_{i,:}\right\| + \left\|\W_{i,:} - \W_{j,:}\right\| + \left\|\W_{j,:} - \frac{1}{\sqrt{n_j}}\W^\star_{j,:}\right\|\\
&= \left\|\W_{i,:} - \W_{j,:}\right\| + \frac{1}{\sqrt{n_i}}\left\|\W^\star_{i,:} -\tW_{i,:}\right\| + \frac{1}{\sqrt{n_j}}\left\|\tW_{j,:} - \W^\star_{j,:}\right\|
\end{align*}
Therefore, since $\W^\star$ is orthogonal, we have
\begin{align*}
||\W_{i,:} - \W_{j,:}|| &\geq \sqrt{\frac{1}{n_i} + \frac{1}{n_j}} - \sqrt{\frac{3k^2\sqrt{3n\epsilon}}{n_i}} - \sqrt{\frac{3k^2\sqrt{3n\epsilon}}{n_j}}\\
&\geq \sqrt{\frac{2}{n}} - \sqrt{12}k(3n\epsilon)^{1/4}.
\end{align*}
This proves the result.

\end{proof}

\begin{corollary}\label{thm:evec_sep_2}
Take $\V\in \R^{n\times k}$ and let $e_1, e_k$ be respectively the smallest and largest eigenvalues of $\V^\top \V$. Suppose $\W\in\R^{k\times k}$ satisfies,
\begin{align*}
\max_{i\in[n]}\left\{\min_{l\in[k]} ||\V_{i,:}-\W_{l,:}||\right\} = \epsilon.
\end{align*}
Then, provided $\epsilon < \sqrt{e_1}(3nk^2)^{-1}$, we have
\begin{align*}
\min_{i,j\in[k], i\not = j} ||\W_{i,:}-\W_{j,:}|| \geq e_k^{-1/2}\left(\sqrt{\frac{2}{n}} - \sqrt{12}k(3ne_1^{-1/2}\epsilon)^{1/4}\right).
\end{align*}
\end{corollary}

\begin{proof}
If $e_1=0$ then there is nothing to show. So assume that $e_1 > 0$ and let $\Sigma = (\V^\top \V)^{-1}$, so that $\V\Sigma^{1/2}$ is orthonormal. Now take any $i \in [n], l\in[k]$, then
\begin{align*}
||\V_{i,:}\Sigma^{1/2} - \W_{l,:}\Sigma^{1/2}||^2 \leq \frac{1}{e_1}||\V_{i,:} - \W_{l,:}||^2.
\end{align*}
Therefore $\W\Sigma^{1/2}$ satisfies,
\begin{align*}
\max_{i\in[n]}\left\{\min_{l\in[k]} ||\V_{i,:}\Sigma^{1/2}-\W_{l,:}\Sigma^{1/2}||\right\} \leq e_1^{-1/2}\epsilon.
\end{align*}
Thus we can apply Proposition~\ref{thm:evec_sep_1} to see that for any $i,j\in[k], i\not = j$, we have
\begin{align*}
||\W_{i,:}\Sigma^{1/2} - \W_{j,:}\Sigma^{1/2}|| \geq \sqrt{\frac{2}{n}} - \sqrt{12}k(3ne_1^{-1/2}\epsilon)^{1/4}\\
\Rightarrow  ||\W_{i,:} - \W_{j,:}|| \geq e_k^{-1/2}\left(\sqrt{\frac{2}{n}} - \sqrt{12}k(3ne_1^{-1/2}\epsilon)^{1/4}\right).
\end{align*}
\end{proof}

By viewing the matrix $\W$ in the above two results as the centroids from a clustering of the rows of $\V$, these effectively place bounds on the distances between cluster centroids. In particular, if the rows of $\V$ are well clusterable in the sense that they all lie very close to their nearest centroid, then the clusters cannot be too close to one another. These results hold precisely because the number of dimensions is equal to (or fewer than) the number of clusters. If we try to extend this to the case where the number of dimensions is greater than the number of clusters, then the matrix of cluster centroids, $\W$, cannot have full column rank. This means that if $\V$ has full column rank, then it cannot be arbitrarily well clusterable by $\W$. We first formalise the above for the case where $\V$ has orthonormal columns, and then extend it to the general case.

\begin{corollary}\label{thm:evec_sep_3}
Let $\V\in\R^{n\times {(k+1)}}$ have orthonormal columns and suppose that $\W\in\R^{k\times (k+1)}$ satisfies
\begin{align*}
\min_{i\in[n]}\left\{\max_{l\in[k]} ||\V_{i,:} - \W_{l,:}||\right\} = \epsilon.
\end{align*}
Then $\epsilon \geq (3n(k+1)^2)^{-1}$.
\end{corollary}

\begin{proof}
Suppose that the result does not hold, i.e., that $\epsilon < (3n(k+1)^2)^{-1}$. Then let $\tW\in\R^{(k+1)\times (k+1)}$ have as its first $k$ rows the rows of $\W$, and as its last row $\W_{l,:}$ where $\W_{l,:}$ is within $\epsilon$ distance of at least two rows of $\V$. Such a $\l\in[k]$ must exist since $\V$ has at least $k+1$ unique rows. Then $\tW$ satisfies the conditions of Proposition~\ref{thm:evec_sep_1}. This would imply that $||\tW_{l,:} - \tW_{k+1,:}|| > 0$, a contradiction. Therefore $\epsilon \geq (3n(k+1)^2)^{-1}$.
\end{proof}

As before, the result can be extended to the case where the columns of $\V$ need not be orthogonal.

\begin{corollary}\label{thm:evec_sep_4}
Let $\V\in\R^{n\times {(k+1)}}$ have full column rank and let $e_1 > 0$ be the smallest eigenvalue of $\V^\top\V$. Suppose that $\W\in\R^{k\times (k+1)}$ satisfies
\begin{align*}
\min_{i\in[n]}\left\{\max_{l\in[k]} ||\V_{i,:} - \W_{l,:}||\right\} = \epsilon.
\end{align*}
Then $\epsilon \geq \sqrt{e_1}(3n(k+1)^2)^{-1}$.
\end{corollary}

\begin{proof}
The proof is exactly analogous to the above proof, where the contradiction now arises by applying Corollary~\ref{thm:evec_sep_2} to the extended matrix $\tW$.
\end{proof}

These final two results will be useful for deriving lower bounds on the eigenvalues of graph Laplacians, which are presented in the following subsection.

\subsection{Eigenvalue Bounds for Graph Laplacians}\label{sec:eval_bounds}

The bounds presented in the previous subsection deal with the structure of the data within the Laplacian eigenvector representation. These can be used to ensure recovery of the maximum margin clustering solution, as discussed in the next subsection. However, no attention has yet been given to selecting the number of clusters. 
It is generally understood that the relative values of the eigenvalues of graph Laplacians might be useful in determining the number of clusters in a data set~\citep{Luxburg2007}. In this section we derive upper and lower bounds on these eigenvalues. These bounds will be used later to show that the number of components of $\Ll$ can be consistently estimated using the eigenvalues of the graph Laplacians computed from truncations of an increasing sample arising from an assumed underlying probability distribution.
The bounds are simply derived, and similar for the different Laplacian matrices. For completeness, we state them all explicitly.

\begin{lemma}\label{thm:eval_bounds_1}
Let $\X = \{\x_1, ..., \x_n\}$ and let $\C_1, ..., \C_k$ be a partition of $\X$. For each $l \in [k]$ suppose that $\C_l$ is connected at distance $\delta_l$. For each $l\in [n]$ let $e_l$ be the $l$-th eigenvalue of the unnormalised Laplacian of the graph $\G = (\X, K_\sigma)$. Then,
\begin{align*}
\sum_{l=1}^k e_l &\leq nk \max_{m\in[k]}K_\sigma (d(\C_m, \X\setminus \C_m)),\\
e_{k+1}&\geq\min_{l,m\in[k]} \frac{K_{\sigma}(\delta_l)}{9n^2(k+1)^4} - n^3kK_\sigma(d(\C_m, \X\setminus \X_m)).
\end{align*}
\end{lemma}

\begin{proof}
Let $\U$ be the eigenvectors of the unnormalised Laplacian of $\G$.
The upper bound on the sum of the first $k$ eigenvalues follows from the beginning of the proof of Lemma~\ref{thm:evec_bound_1}, since $\sum_{l=1}^ke_l = \sum_{l=1}^k\U_{:,l}^\top \L\U_{:,l}$.
Now, by Corollary~\ref{thm:evec_sep_3} we know that $\exists i, j \in [n], l\in[k]$ with $\x_i, \x_j\in\C_l$, such that
\begin{align*}
||\U_{i,:1:(k+1)} - \U_{j,:1:(k+1)}||^2 \geq (3n(k+1)^2)^{-2}.
\end{align*}
But from Lemma~\ref{thm:evec_bound_1} we know that
\begin{align*}
||\U_{i,1:k} - \U_{j,1:k}||^2 \leq \max_{m\in[k]} n^3k \frac{K_\sigma(d(\C_m, \X\setminus\C_m))}{K_\sigma(\delta_l)},
\end{align*}
and thus,
\begin{align*}
(\U_{i,k+1} - \U_{j,k+1})^2 \geq (3n(k+1)^2)^{-2} - \max_{m\in[k]} n^3k \frac{K_\sigma(d(\C_m, \X\setminus\C_m))}{K_\sigma(\delta_l)}.
\end{align*}
Again using~\cite[Proposition 1]{Luxburg2007}, we have
\begin{align*}
e_{k+1} &= \U_{:,k+1}^\top \L \U_{:,k+1} = \frac{1}{2}\sum_{g,h}K_{\sigma}(||\x_g-\x_h||)(\U_{g,k+1}-\U_{h,k+1})^2\\
&\geq K_{\sigma}(\delta_l)(\U_{i,k+1}-\U_{j,k+1})^2\\
&\geq \min_{l,m\in[k]} \frac{K_{\sigma}(\delta_l)}{9n^2(k+1)^4} - n^3kK_\sigma(d(\C_m, \X\setminus \X_m)).
\end{align*}

\end{proof}

\begin{lemma}\label{thm:eval_bounds_2}
Let $\X = \{\x_1, ..., \x_n\}$ and let $\C_1, ..., \C_k$ be a partition of $\X$. For each $l \in [k]$ suppose that $\C_l$ is connected at distance $\delta_l$. For each $l\in [n]$ let $e_l$ be the $l$-th eigenvalue of the normalised Laplacian of the graph $\G = (\X, K_\sigma)$. Then,
\begin{align*}
\sum_{l=1}^k e_l &\leq nk \max_{m\in[k]}K_\sigma (d(\C_m, \X\setminus \C_m)),\\
e_{k+1}&\geq\min_{l,m\in[k]} \frac{K_{\sigma}(\delta_l)}{9n^3(k+1)^4} - n^3kK_\sigma(d(\C_m, \X\setminus \X_m)).
\end{align*}
\end{lemma}

\begin{proof}
Let $\U$ be the eigenvectors of the normalised Laplacian of $\G$. The upper bound on the sum of the first $k$ eigenvalues now follows immediately from the first part of the proof of Lemma~\ref{thm:evec_bound_2}.
Unlike in the previous proof, we cannot use Corollary~\ref{thm:evec_sep_3} since $\D^{-1/2}\U_{:,1:(k+1)}$ is not orthogonal. However, observe that,
\begin{align*}
\min_{\v\in\R^d}\frac{\v^\top\U_{:,1:(k+1)}^\top \D^{-1}\U_{:,1:(k+1)}\v}{\v^\top\v} &= \min_{\v\in\R^d}\frac{\v^\top\U_{:,1:(k+1)}^\top \D^{-1}\U_{:,1:(k+1)}\v}{\v^\top\U^\top_{:,1:(k+1)}\U_{:,1:(k+1)}\v}\\
& \geq \min_{\u\in\R^n} \frac{\u^\top \D^{-1}\u}{\u^\top\u}\\
& = \min_{i\in[n]} \D_{i,i}^{-1}\geq \frac{1}{n}.
\end{align*}
That is, the smallest eigenvalue of $\U_{:,1:(k+1)}^\top \D^{-1}\U_{:,1:(k+1)}$ is at least $n^{-1}$, and so by Corollary~\ref{thm:evec_sep_4} we know there exist $i, j\in[n], l\in[k]$ with $\x_i,\x_j\in\C_l$ such that,
\begin{align*}
||\D_{i,i}^{-1/2}\U_{i,1:(k+1)}-\D_{j,j}^{-1/2}\U_{j,1:(k+1)}||^2 \geq \frac{1}{9}n^{-3}(k+1)^{-4}.
\end{align*}
The rest of the proof is analogous to the previous proof.
\end{proof}

\begin{lemma}\label{thm:eval_bounds_3}
Let $\X = \{\x_1, ..., \x_n\}$ and let $\C_1, ..., \C_k$ be a partition of $\X$. For each $l \in [k]$ suppose that $\C_l$ is connected at distance $\delta_l$. For each $l\in [n]$ let $e_l$ be the $l$-th eigenvalue of the normalised Laplacian of the graph $\G = (\X, K_\sigma)$, but with reflexive edges removed. Then,
\begin{align*}
\sum_{l=1}^k e_l &\leq nk \max_{m\in[k]}\frac{K_\sigma (d(\C_m, \X\setminus \C_m))}{K_\sigma(\delta_m)},\\
e_{k+1}&\geq\min_{l,m\in[k]} \frac{K_{\sigma}(\delta_l)}{9n^3(k+1)^4} - n^3k\frac{K_\sigma(d(\C_m, \X\setminus \X_m))}{K_\sigma(\delta_m)}.
\end{align*}
\end{lemma}

\begin{proof}
The proof is exactly as above, but using the bound from Lemma~\ref{thm:evec_bound_3} instead of Lemma~\ref{thm:evec_bound_2}.
\end{proof}


\subsection{Maximum Margins from Graph Laplacians}\label{sec:mmc}

In this section we provide a detailed derivation of the convergence of spectral clustering to the maximum margin clustering solution. The results for the unnormalised Laplacian, $\L$, and normalised Laplacian, $\Ln$, follow from assumptions AK1--AK3 on the kernel, $K$, and the results from Section~\ref{sec:evec_bounds}. Convergence of the spectral clustering solution arising from the normalised Laplacian of the similarity graph without reflexive edges, $\Lno$, requires stronger assumptions. To ease the analysis, we assume in this case that the kernel function is given simply by $K(x) = \exp(-x^\alpha)$ for some $\alpha > 0$. We also require that the within cluster connectedness is sufficiently below the between cluster separatedness, in terms of the parameter $\alpha$. While in the scenario of level set estimation from an increasing sample, this is not a problem since the within cluster connectivity converges to zero as the sample size increases. In the finite sample setting, however, without this additional requirement, convergence instead occurs as $\sigma \to 0$ and $\alpha \to \infty$, rather than only requiring that $\sigma \to 0$ as in the cases of $\L$ and $\Ln$.

We state the results of this section in a form which is convenient for the consistency analysis given in the following section, where convergence occurs as the sample size, $n$, increases. Broadly speaking the results show that for $\sigma < A\log(Bn^{z})^{-C}$, where $A, B, C>0$ are independent of $n$, and any $z$ sufficiently large, we have that the ratio of within cluster distances to between cluster distances, in the eigenvector representation, are $\mathcal{O}(n^{D-Ez})$ for constants $D, E > 0$. This means that in the finite sample setting, where $n$ is fixed, as $\sigma$ approaches zero (and thus $z$ increases towards $\infty$), the maximum margin clustering solution becomes trivially attainable from the eigenvectors. Once again we investigate each Laplacian matrix separately.

\begin{lemma}\label{thm:mmc_unnorm}
Let $\X = \{\x_1, ..., \x_n\}$ and assume that there is a partition of $\X$ into $\C_1, ..., \C_k$ such that $\min_{m\in[k]} d(\C_m, \X\setminus\C_m) - \max_{l\in[k]} \delta_l = \delta > 0$, where $\C_l$ is connected at distance $\delta_l$ for each $l\in[k]$. For $\sigma > 0$ let $\L$ be the unnormalised Laplacian of the graph $\G = (\X, K_\sigma)$, where $K$ satisfies assumptions AK1--AK3. Let $\U$ have as column the eigenvectors of $\L$. 
Then, provided $0 < \sigma < \delta \log(An^{z/3})^{-1/\alpha}$, where $A$ and $\alpha$ are as in assumption AK3, and $z$ satisfies $n^{z-15} \geq 81k^{15}$, we have
\begin{align*}
\max_{\substack{i,j\in[n],l\in[k]:\\ \x_i, \x_j\in \C_l}} ||\U_{i,1:k} - \U_{j,1:k}|| & \leq \left(\frac{k^3}{n^{z-9}}\right)^{1/6},\\
\min_{\substack{i,j\in[n],l\in[k]:\\ \x_i\in \C_l, \x_j\not\in \C_l}} ||\U_{i,1:k} - \U_{j,1:k}|| & \geq \sqrt{\frac{2}{n}} -6\left(\frac{k^{27}}{n^{z-15}}\right)^{1/24}.
\end{align*}
%
%
%
\end{lemma}

\begin{proof}
First, combining lemma~\ref{thm:evec_bound_1} with assumption AK3, we get
\begin{align*}
\max_{i,j\in[n],l\in[k]: \x_i, \x_j\in \C_l} ||\U_{i,1:k} - \U_{j,1:k}||^2 &\leq \max_{m\in[k]} n^3 k \frac{K_\sigma(d(\C_m,\X\setminus \C_m))}{K_\sigma(\delta_l)}
\leq An^3 k \exp\left(-\left(\frac{\delta}{\sigma}\right)^\alpha\right)\\
&\leq An^3k\exp\left(-\log(An^{z/3})\right)
= \frac{k}{n^{(z-9)/3}},
\end{align*}
as required.
Now let $i_1, ..., i_k$ be such that $\x_{i_l}\in\C_l$ for each $l\in[k]$. Then for $\sigma$ below the assumed upper bound the matrix $[\U_{i_1,1:k}, ..., \U_{i_k,1:k}]$ satisfies the conditions on the matrix $\W$ in proposition~\ref{thm:evec_sep_1}. Therefore,
\begin{align*}
\min_{l,m\in[k]} ||\U_{i_l,1:k} - \U_{i_m,1:k}|| &= \min_{\substack{i,j\in[n],l,m\in[k]:\\ \x_i\in \C_l, \x_j\in \C_m}} ||\U_{i_l,1:k} - \U_{i,1:k} + \U_{i,1:k} - \U_{j,1:k} + \U_{j,1:k} - \U_{i_m,1:k}||\\
&\leq \min_{\substack{i,j\in[n],l,m\in[k]:\\ \x_i\in \C_l, \x_j\in \C_m}} ||\U_{i,1:k} - \U_{j,1:k}|| + 2\sqrt{An^3k}\exp\left(-\frac{1}{2}\left(\frac{\delta}{\sigma}\right)^\alpha\right)\\
\Rightarrow \min_{\substack{i,j\in[n],l\in[k]:\\ \x_i\in \C_l, \x_j\not\in \C_l}} ||\U_{i,1:k} - \U_{j,1:k}|| &\geq \sqrt{\frac{2}{n}} - \sqrt{12}k\left(3n\sqrt{An^3k}\exp\left(-\delta^\alpha/2\sigma^\alpha\right)\right)^{1/4}\\
& \hspace{20pt} - 2\sqrt{An^3k}\exp\left(-\frac{1}{2}\left(\frac{\delta}{\sigma}\right)^\alpha\right).
\end{align*}
Now, it is simple to verify that the assumption on the value of $\sigma$ ensures that the second two terms on the right hand side above sum to less than $4k\left(3n\sqrt{An^3k}\exp\left(-\delta^\alpha/2\sigma^\alpha\right)\right)^{1/4}$, for any $n\geq 2$. Therefore,
\begin{align*}
\min_{\substack{i,j\in[n],l\in[k]:\\ \x_i\in \C_l, \x_j\not\in \C_l}} ||\U_{i,1:k} - \U_{j,1:k}||&\geq \sqrt{\frac{2}{n}} - 4k\left(3n\sqrt{An^3k}\exp\left(-\delta^\alpha/2\sigma^\alpha\right)\right)^{1/4}\\
&\geq \sqrt{\frac{2}{n}} -6k^{9/8}n^{-(z-15)/24},
\end{align*}
with the last step coming from simple rearrangement after substituting in the upper bound for $\sigma$.

\end{proof}

\begin{remark}
The above result assumes that the within cluster connectedness is strictly less than the between cluster separatedness. This occurs with probability one if $\X$ is seen as a sample of realisations of a continuous random variable on $\R^d$.
\end{remark}

Stating the bounds in the above theorem in terms of $n$ is convenient for the theorey presented in the next subsection. However, it can be seen directly from the above that
$$
\max_{\substack{i,j\in[n],l\in[k]:\\ \x_i, \x_j\in \C_l}} ||\U_{i,1:k} - \U_{j,1:k}||\xrightarrow{as \ \sigma \to 0^+} 0,
$$
while the lower bound on the term
$\min_{\substack{i,j\in[n],l\in[k]:\\ \x_i\in \C_l, \x_j\not\in \C_l}} ||\U_{i,1:k} - \U_{j,1:k}||$ converges to $\sqrt{2/n}$ as~${\sigma \to 0^+}$. The maximum margin clustering solution is therefore trivially obtained from the limit of the spectral clustering solution using the unnormalised Laplacian.
The corresponding result to Lemma~\ref{thm:mmc_unnorm} for the normalised Laplacian requires only minor modifications.

\begin{lemma}\label{thm:mmc_norm}
Let $\X = \{\x_1, ..., \x_n\}$ and assume that there is a partition of $\X$ into $\C_1, ..., \C_k$ such that $\min_{m\in[k]} d(\C_m, \X\setminus\C_m) - \max_{l\in[k]} \delta_l = \delta > 0$, where $\C_l$ is connected at distance $\delta_l$ for each $l\in[k]$. For each $\sigma > 0$ let $\Ln$ be the normalised Laplacian of the graph $\G = (\X, K_\sigma)$, where $K$ satisfies assumptions AK1--AK3, and let $\D$ be the corresponding degree matrix. Let $\U$ be the eigenvectors of $\Ln$. 
Then, for $0 < \sigma < \delta \log(An^{z/3})^{-1/\alpha}$, where $A$ and $\alpha$ are as in assumption AK3, and $z$ satisfies $n^{z-18} \geq 81k^{15}$, we have
\begin{align*}
\max_{\substack{i,j\in[n],l\in[k]:\\ \x_i, \x_j\in \C_l}} ||\D^{-1/2}_{i,i}\U_{i,1:k} - \D^{-1/2}_{j,j}\U_{j,1:k}|| & \leq \left(\frac{k^3}{n^{z-9}}\right)^{1/6},\\
\min_{\substack{i,j\in[n],l\in[k]:\\ \x_i\in \C_l, \x_j\not\in \C_l}} ||\D_{i,i}^{-1/2}\U_{i,1:k} - \D^{-1/2}_{j,j}\U_{j,1:k}|| & \geq \sqrt{\frac{2}{n}} -6\left(\frac{k^{27}}{n^{z-18}}\right)^{1/24}.
\end{align*}
%
%
%
\end{lemma}

\begin{proof}
The proof is similar to before, except that now we cannot use Proposition~\ref{thm:evec_sep_1}, since $\D^{-1/2}\U_{:,1:k}$ is not orthogonal. As in the proof of Lemma~\ref{thm:eval_bounds_2}, we know that 
%
%
the smallest eigenvalue of $\U_{:,1:k}^\top \D^{-1}\U_{:,1:k}$ is at least $n^{-1}$. Similarly, since all diagonal elements of $\D$ are at least one, the largest eigenvalue is at most 1. We now have, therefore, using Corollary~\ref{thm:evec_sep_2}, that
\begin{align*}
\min_{\substack{i,j\in[n],l\in[k]:\\ \x_i\in \C_l, \x_j\not\in \C_l}} ||\D^{-1/2}_{i,i}\U_{i,1:k} - \D^{-1/2}_{j,j}\U_{j,1:k}|| &\geq \sqrt{\frac{2}{n}} - 4k\left(3n^{3/2}\sqrt{An^3k}\exp\left(-\delta^\alpha/2\sigma^\alpha\right)\right)^{1/4},
\end{align*}
with a slightly higher power of $n$ in the second term on the right hand side than we had before. 
The rest of the proof is exactly analogous to the previous proof.
\end{proof}

In the above two results, we obtained explicit lower bounds on the between cluster distances within the eigenvector representations. Combining this with the fact that the within cluster distances converge to zero, the recovery of the maximum margin clustering solution from these eigenvectors is therefore immediate. 
In the case where the diagonal elements of the affinity matrix are set to zero, however, we are not able to place a lower bound on the between cluster distances within the eigenvectors. We therefore only show that in this case the within cluster distances converge to zero at a much faster rate than the between cluster distances, as $\sigma$ tends to zero. This is a slightly weaker result, but still ensures the maximum margin solution arises trivially from the eigenvectors of the normalised Laplacian.

As mentioned previously, we also require additional assumptions.
We study the case where the kernel takes the explicit form of $K(x) = \exp(-x^\alpha)$ for some $\alpha > 0$. This allows us to generalise assumption AK3 to allow not only ratios of individual kernel values. We also require a stricter assumption on the relationship between the within cluster connectedness and between cluster sepation than was used previously. This is made explicit in the statement of the result below.

\begin{lemma}\label{thm:mmc_norm0}
Let $\X = \{\x_1, ..., \x_n\}$ and assume that there is a partition of $\X$ into $\C_1, ..., \C_k$ such that $\min_{l\in[k]} |\C_l| \geq 2$, $\min_{m\in[k]}d(\C_m, \X\setminus\C_m)^\alpha - 3\max_{l\in[k]}\delta_l^\alpha = \delta > 0$, where $\C_l$ is connected at distance $\delta_l$ for each $l\in[k]$ and $\alpha$ is given in the formulation of the kernel, $K$, which follows. For each $\sigma > 0$ let $\Lno$ be the normalised Laplacian of the graph $\G = (\X, K_\sigma)$, but with reflexive edges removed, where $K(x) = \exp(-x^\alpha)$ for some $\alpha > 0$, and let $\D$ be the corresponding degree matrix. Let $\U$ be the eigenvectors of $\Lno$. Then, for $0 < \sigma < \delta^{1/\alpha}\log\left(13^8k^{9}n^z\right)^{-1/\alpha}$, where $z \geq 10$, we have
\begin{align*}
\max_{\substack{i,j,g,h\in[n],l,m\in[k]:\\ \x_i, \x_j\in \C_l\\ \x_g, \x_h\in \C_m}} \frac{||\D^{-1/2}_{i,i}\U_{i,1:k} - \D^{-1/2}_{j,j}\U_{j,1:k}||}{||\D^{-1/2}_{g,g}\U_{g,1:k} - \D^{-1/2}_{h,h}\U_{h,1:k}||} & \leq n^{(4-z)/2}.
\end{align*}
%
%
%
\end{lemma}

\begin{proof}
First let $\delta_\star = \max_{l\in[k]}\delta_l$.
Combining Lemma~\ref{thm:evec_bound_3} with assumptions on $K$, we get
\begin{align*}
\max_{i,j\in[n],l\in[k]: \x_i, \x_j\in \C_l} ||\U_{i,1:k} - \U_{j,1:k}||^2 &\leq \max_{m\in[k]} n^3 k \frac{K_\sigma(d(\C_m,\X\setminus \C_m))}{K_\sigma(\delta_l)K_\sigma(\delta_m)}\\
&\leq n^3 k \exp\left(-\frac{\delta_\star^\alpha +\delta}{\sigma^\alpha}\right).
\end{align*}
By assumption, points within any clusters containing at least two points are within $\delta_\star$ of their nearest neighbours. The lower bound on the diagonals of the degree matrix is therefore now $K_\sigma(\delta_\star)$, instead of 1 as before. The largest eigenvalue of $\U_{:,1:k}^\top\D^{-1}\U_{:,1:k}^\top$ is thus at most $\max_{i\in[n]}\D^{-1}_{i,i} \leq K_\sigma(\delta_\star)^{-1} = \exp(\delta_\star^\alpha/\sigma^\alpha)$. The smallest eigenvalue is again at least $n^{-1}$. Therefore in this case we have, using Lemma~\ref{thm:evec_sep_2}, and after simple rearranging,
\begin{align*}
\min_{\substack{i,j\in[n],l\in[k]:\\ \x_i\in \C_l, \x_j\not\in \C_l}} ||\D^{-1/2}_{i,i}\U_{i,1:k} - \D^{-1/2}_{j,j}\U_{j,1:k}|| &\geq 
\sqrt{\frac{2}{n}}\exp\left(-\frac{\delta_\star^\alpha}{2\sigma^\alpha}\right) - \sqrt{12}k^{9/8}n^{3/4}\exp\left(-\frac{5\delta_\star^\alpha + \delta}{8\sigma^\alpha}\right)\\
& \hspace{20pt} - 2n^{3/2}k^{1/2}\exp\left(-\frac{\delta_\star^\alpha+\delta}{2\sigma^\alpha}\right).
\end{align*}
Now, the upper bound on $\sigma$ ensures that the above difference is at least $n^{-1/2}\exp(-\delta_\star^\alpha/2\sigma^\alpha)$. Putting these together, we get, after simplification,
\begin{align*}
\max_{\substack{i,j,g,h\in[n],l,m\in[k]:\\ \x_i, \x_j\in \C_l\\ \x_g, \x_h\in \C_m}} \frac{||\D^{-1/2}_{i,i}\U_{i,1:k} - \D^{-1/2}_{j,j}\U_{j,1:k}||}{||\D^{-1/2}_{g,g}\U_{g,1:k} - \D^{-1/2}_{h,h}\U_{h,1:k}||} &\leq \frac{n^{3/2}k^{1/2}\exp\left(-\frac{\delta_\star^\alpha+\delta}{2\sigma^\alpha}\right)}{n^{-1/2}\exp\left(-\frac{\delta_\star^\alpha}{2\sigma^\alpha}\right)}\\
&\leq n^2\exp\left(-\frac{\delta}{2\sigma^\alpha}\right) \leq n^{(4-z)/2}.
\end{align*}
%
%
\end{proof}

\begin{remark}
The stricter assumption in the previous theorem may appear to be stated only for the convenience of making the theorem hold, rather than being practically relevant. However, consider that if $0 < a < b$ then there is a $\alpha > 0$ s.t. $b^\alpha > 3a^\alpha$. From a practical point of view, therefore, if we were to determine sequences of similarities using $\exp(-||\x_i-\x_j||^\alpha/\sigma^\alpha)$ where $\sigma \to 0$ and $\alpha \to \infty$, then the convergence of the spectral clustering solution to the maximum margin solution would hold under the same assumptions as in Lemmas~\ref{thm:mmc_unnorm} and~\ref{thm:mmc_norm}. In addition, in the situation where the within cluster connectedness decreases appropriately towards zero as $n$ increase, while the between cluster separation is bounded below, then for any fixed value of $\alpha$ the above theorem takes effect. This will be relevant in the following section, where we study the behaviour of the spectral clustering solution applied to a truncated sample, as the size of the sample increases.
\end{remark}

\subsection{Consistently Estimating Level Set Components using Spectral Clustering}\label{sec:levsets}

In this section we study the estimation of level sets, and their components, using spectral clustering. Unlike existing work on this problem~\citep{pelletier2011operator}, we study multiple versions of spectral clustering. Specifically, those arising from the relaxations of the Ratio Cut proplem, and two similar versions of the Normalised Cut problem. In addition, we consider kernels with potentially unbounded support, and we also consider what we believe is a more desirable context where the scaling parameter is decreased towards zero as the sample size increases. This means that our estimation procedure requires weaker assumptions than the existing theory. Importantly the minimum distance between components of the target level set need not be known. Furthermore, the requirements on the rate of decrease of the scaling parameter which we require admits the asymptotically optimal mean integrated squared error (MISE) rate for the related problem of kernel density estimation. This adds a superficial (and minor computational) benefit which is that the same similarities used in the spectral clustering algorithm may be used to estimate the density, and hence level set as well.

Suppose that $X_1, X_2, ...$ is a sequence of i.i.d. random variables on $\R^d$ with distribution admitting density $p$, which satisfies assumptions A1 and A2 for level $\lambda > 0$. Suppose also that the kernel, $K$, satisfies assumptions AK1--AK3.
We begin by deriving some connectivity properties of the elements of $X_1, X_2, ..., X_n$ which lie in an estimate of $\Ll$, say $\widehat\Ll^{(n)}$, in relation to its components. To that end, if the level set $\Ll$ has $c$ components, then we will use $\ell(\lambda, 1), ..., \ell(\lambda, c)$ to denote these components, listed according to any arbitrary ordering. What we show is that, with probability one, points arising in a shrinking sequence of neighbourhoods of one of the components are connected at distances approximately $\mathcal{O}(\sigma_n)$, where $\{\sigma_n\}_{n=1}^\infty$ is an appropriately chosen sequence of scaling parameters. Furthermore, with probability one, no points outside these neighbourhoods of the components of $\Ll$ are included in $\widehat\Ll^{(n)}$, for large values of $n$. This second point ensures that the level set itself is consistently estimated by taking shrinking neighbourhoods around the points in $\widehat\Ll^{(n)}$. We go on to show that the sequence $\{\sigma_n\}_{n=1}^\infty$ can be chosen so that the first $c$ eigenvectors of the Laplacian matrices of $\G = (\widehat\Ll^{(n)}, K_{\sigma_n})$ trivially expose the separation of $\widehat\Ll^{(n)}$ into the subsets falling in the shrinking neighbourhoods of the components of $\Ll$, mentioned above. Finally, the same sequence of scaling parameters leads to the eigenvalues of these Laplacians allowing for consistent estimation of $c$.

\begin{lemma}\label{thm:consistency}
Let $X_1, X_2, ...$ be an i.i.d. sequence of random variables on $\R^d$, $d\geq 2$, with density $p$ satisfying assumptions A1--A3 for level $\lambda > 0$. Let $K:\R^+\to\R^+$ satisfy assumptions AK1--AK3. Let $\{\sigma_n\}_{n=1}^\infty$ and $\{S_n\}_{n=1}^\infty$ be null sequences satisfying, for all large $n$ and any fixed $\xi > 1$,
\begin{align*}
\max\left\{n^{-1/d}\log(n), \left(     \frac{\log(n^{1/d}/\log(n))}{nS_n^{\xi}}     \right)^{1/d}\right\} < \sigma_n < \min\{\left(\log\log n\right)^{-(1+\xi)/\alpha\xi}, S_n^\xi\}.
\end{align*}
%
%
Then there exists a sequence $\{a_n\}_{n=1}^\infty$ with $a_n = O(S_n)$ such that if we define, for each $k \in [c]$ and $n \in \N$,
\begin{align*}
\widehat \llk^{(n)} = \left\{X_j \bigg| j \leq n, \frac{c_K}{n\sigma_n^d}\sum_{i=1}^n K(||X_i-X_j||/\sigma_n) > \lambda - S_n, X_j \in \B_{a_n}(\llk)\right\},
\end{align*}
then with probability one, for all large $n$ we have,
\begin{enumerate}
\item $\widehat \llk^{(n)}$ is connected at distance $a_n$, for all $k\in [c]$,
\item $\llk \subset \B_{\frac{a_n}{2}}(\widehat \llk^{(n)}) \subset \B_{a_n}(\llk)$ for all $k\in[c]$,
\item $\max \left\{\frac{c_K}{n\sigma_n^d}\sum_{i=1}^n K(||X_i-X_j||/\sigma_n)\bigg| j\leq n, X_j \not \in \bigcup_{k\in [c]}\widehat \llk^{(n)} \right\} \leq \lambda - S_n$.
\end{enumerate}
%
%
\end{lemma}

\begin{proof}
The assumptions on $p$, $K$ and $\{\sigma_n\}_{n=1}^\infty$ satisfy the conditions for the uniform convergence of the density estimator,
\begin{align*}
\hat p_n(\x) := \frac{c_K}{n \sigma_n^d} \sum_{i=1}^n K\left(\frac{||\x - X_i||}{\sigma_n}\right).
\end{align*}
In particular, there exists a constant $B$ such that with probability one, the following inequality holds for all large $n$,
%
%
\begin{align*}
\sup_{\x\in \R^d}\left|\hat p_n(\x) - p(\x)\right| &\leq B\left(\sqrt{-\frac{\log(\sigma_n)}{n\sigma_n^d}}+\sigma_n\right)=:g_n,
\end{align*}
where we have combined the result of~\cite{gine2002rates} with the standard $\mathcal{O}(\sigma_n)$ bias of the kernel estimator of a density with bounded first derivative. It is easy to check that the assumptions on the sequences $\{\sigma_n\}_{n=1}^\infty$ and $\{S_n\}_{n=1}^\infty$ ensure that $g_n = o(S_n)$.

Now take any $\frac{\xi}{1+\xi} < \epsilon < 1$. Then for all $n$ large enough we have the above as well as the following,
\begin{description}
\item[N1:]$g_n + \kappa\sigma_n^{1-\epsilon}  < S_n$, where $\kappa$ is as in Assumption A1.
\item[N2:] $S_n + g_n < \gamma$, where $\gamma$ is as in assumption A2.
\item[N3:] $\frac{\sigma_n^d}{c_K}(\lambda - g_n - S_n) - A\exp(-\sigma_n^{-\alpha\epsilon}) \geq 1/n$, where $A$ and $\alpha$ are as in Assumption AK3.
\end{description}

\noindent
Note that N3 is ensured by the upper bound on $\sigma_n$. We now define the sequence $a_n = 2(C(S_n+g_n) + \sigma_n^{1-\epsilon})$. Then by N1 above we have $a_n = O(S_n)$. We now go on to show that $\{a_n\}_{n=1}^\infty$ satisfies the three results stated in the theorem.
Combining N1 and N2 above, we have,
\begin{align*}
\min\left\{p(X_j) \big| j\leq n, \hat p_n(X_j) > \lambda - S_n \right\} &> \lambda - S_n - g_n\\
\Rightarrow \max \left\{d(X_j, \Ll)\big| j\leq n, \hat p_n(X_j) > \lambda - S_n\right\} &< C(S_n+g_n)\\
\end{align*}
As a result, every element of $\{X_1, ..., X_n\}$ whose estimated density is above $\lambda - S_n$ is within $C(S_n+g_n)$ of a component of $\Ll$. Result 3. in the statement of the theorem follows immediately. Furthermore, take any $\w \in \Ll$. Then, $\hat p_n(\w) \geq \lambda - g_n$, and so
\begin{align*}
\lambda - g_n \leq \frac{c_k}{n\sigma_n^d}\sum_{i=1}^n K\left(\frac{||\w - X_i||}{\sigma_n}\right) &\leq \frac{c_k}{n\sigma_n^d}\sum_{i: ||\w-X_i|| < \sigma_n^{1-\epsilon}} K\left(\frac{||\w - X_i||}{\sigma_n}\right) + \frac{c_K}{\sigma_n^d}K\left(\sigma_n^{-\epsilon}\right)\\
\Rightarrow  \frac{n\sigma_n^d(\lambda - g(n))}{c_K} - nA\exp(-\sigma_n^{-\alpha\epsilon}) &\leq \sum_{i: ||\w-X_i|| < \sigma_n^{1-\epsilon}}K\left(\frac{||\w - X_i||}{\sigma_n}\right)\\
& \leq \left|\{X_1, ..., X_n\} \cap \B_{\sigma_n^{1-\epsilon}}(\w)\right|,
\end{align*}
and so by N3, $\left|\{X_1, ..., X_n\} \cap \B_{\sigma_n^{1-\epsilon}}(\w)\right| \geq 1$. As a result, for any $\w \in \llk$, for some $k\in [c]$, there exists $j \in [n]$ such that $d(\w, X_j) < \sigma_n^{1-\epsilon}$, and so $p(X_j) \geq p(\w)+\kappa \sigma_n^{1-\epsilon} \Rightarrow \hat p_n(X_j) > \lambda - \kappa \sigma_n^{1-\epsilon} - g_n > \lambda - S_n \Rightarrow X_j \in \widehat \llk^{(n)}$.

Notice also, from above, that, since there is no $j\in [n]$ s.t. $\hat p(X_j) > \lambda - S_n$ and $d(X_j, \Ll) \geq C(S_n + g_n)$, we have
\begin{align*}
\widehat \llk^{(n)} \subset \B_{C(S_n+g_n)}(\llk).
\end{align*}
Now take any $\w\in \B_{C(S_n+g_n)}(\llk)$. Then $d(\w, \llk) < C(S_n + g_n) \Rightarrow d(\w, \widehat \llk^{(n)}) < C(S_n + g_n)+\sigma_n^{1-\epsilon}$, since every point in $\llk$ is within $\sigma_n^{1-\epsilon}$ of some $X_j, j\in [n]$ with $\hat p_n(X_j) > \lambda - S_n$.
Since $\w$ was arbitrary, we thus have that $\widehat \llk^{(n)}$ is connected at distance $2(C(S_n+g_n)+\sigma_n^{1-\epsilon}) = a_n$. This proves result 1. in the theorem.

Result 2. also follows immediately from above, since we have established that every $\w \in \llk$ lies in $\B_{\sigma_n^{1+\epsilon}}(\widehat \llk^{(n)}) \subset \B_{\frac{a_n}{2}}(\widehat \llk^{(n)})$, and also that $\widehat \llk^{(n)} \subset \B_{C(S_n+g_n)}(\llk) \subset \B_{\frac{a_n}{2}}(\llk)$.

\end{proof}

The first and second results in the above Lemma ensures that points in the sequence $X_1, X_2, ...$ which fall in the same level set component are closely connected for large values of $n$, whereas the second and third results ensure that, with probability one, if $X_i\in\llk$ and $X_j\not \in \llk$, for some $k$, then for all large $n$, either $X_j \not \in \widehat\Ll^{(n)}$ or there is no subset of $\widehat\Ll^{(n)}$ containing both $X_i$ and $X_j$ which is closely connected.



Next we show that the degree of connectedness and separation of points falling in each of the level set components is sufficient for spectral clustering to allow trivial recovery of the desired partition, almost surely, as $n\to\infty$. As always, we cover the different Laplacians separately for completeness.

\begin{lemma}\label{thm:evec_consistency_1}
Let $X_1, X_2, ...$ be an i.i.d. sequence of random variables on $\R^d$, $d\geq 2$, with density $p$ satisfying assumptions A1--A3 for level $\lambda > 0$, and suppose that the number of components of $\Ll$ is $c < \infty$. Let $K:\R^+\to\R^+$ satisfy assumptions AK1--AK3. Let $0 < \epsilon < 1$ be fixed and let $\{\sigma_n\}_{n=1}^\infty$ be a null sequence satisfying, for all large $n$,
\begin{align*}
n^{-1/(d+1)} < \sigma_n < log(n)^{-\nu},
\end{align*}
for any $\nu > \alpha^{-1}$, where $\alpha$ is as in assumption AK3.
For each $n \in N$ let
\begin{align*}
\widehat \Ll^{(n)} = \left\{X_j \big| j\leq n, \frac{c_K}{n\sigma_n^d}\sum_{i=1}^n K(||X_i - X_j||/\sigma_n) > \lambda - D\sigma_n^{1-\epsilon}\right\},
\end{align*}
for any fixed $D>0$.
 Let $\L^{(n)}$ be the unnormalised Laplacian of the graph with vertices $\widehat \Ll^{(n)}$ and edge weights determined using $K_{\sigma_n}$, and let $\U^{(n)}$ be its eigenvectors. Now, for each $i,j,n\in\N$, define
 \begin{align*}
 d_{ij}^{(n)} = \left\{\begin{array}{ll}
 ||\U^{(n)}_{(i),1:c}-\U^{(n)}_{(j),1:c}||, \mbox{ if $X_i, X_j \in \widehat\Ll^{(n)}$}; X_i = \widehat\Ll^{(n)}_{(i)}, X_j = \widehat\Ll^{(n)}_{(j)}\\
 1, \mbox{ otherwise.}
 \end{array}\right.
 \end{align*}
Implicitly if either $i>n$ or $j>n$ then $d_{ij}^{(n)} = 1$. Then, with probability one, for all fixed $i, j \in \N, k\in[c]$ we have,
\begin{align*}
X_i, X_j \in \llk \Rightarrow &\lim_{n\to\infty} nd_{ij}^{(n)} = 0,\\
X_i \in\llk, X_j \not\in \llk \Rightarrow &\lim_{n\to\infty} nd_{ij}^{(n)} = \infty.
\end{align*}
\end{lemma}

\begin{proof}
First note that if we set, for each $n\in \N$, $S_n = D\sigma_n^{1-\epsilon}$, then $\{\sigma_n\}_{n=1}^\infty$ and $\{S_n\}_{n=1}^\infty$ satisfy the requirements in Lemma~\ref{thm:consistency}, if we choose $1 < \xi < \frac{1}{1-\epsilon}$.
Now,
since $\Ll$ is closed and has finitely many components, we know that there exists a $\Delta > 0$ such that for all $k,l\in[c]$, $\B_\Delta\llk \cap \B_\Delta \ell(\lambda,l) = \emptyset$. Combining the results of Lemma~\ref{thm:consistency}, it is straightforward to verify that there exists a $M>0$ independent of $n$ such that, with probability one, for all large $n$, we have
\begin{enumerate}
\item $\widehat\llk^{(n)} := \widehat \Ll^{(n)} \cap \B_{M\sigma_n^{1-\epsilon}}(\llk)$ is connected at distance $M\sigma_n^{1-\epsilon}$ for all $k\in [c]$.
\item $\llk \subset \B_{M\sigma_n^{1-\epsilon}}(\widehat\llk^{(n)}) \subset \B_{2M\sigma_n^{1-\epsilon}}(\llk)$.
\item For all $k, l \in [c]$, $k\not=l$, we have $d(\widehat \llk^{(n)}, \widehat{ \ell(\lambda, l)}^{(n)}) \geq \Delta$.
\item $\{\widehat{\ell(\lambda, 1)}^{(n)}, ..., \widehat{\ell(\lambda, c)}^{(n)}\}$ is a partition of $\widehat \Ll^{(n)}$.
\end{enumerate}
For any $z \in \R$ we clearly have, for all large $n$, that $\sigma_n < (\Delta - M\sigma_n^{1-\epsilon})\log(An^{z/3})^{-1/\alpha}$, since $\sigma_n < \log(n)^{-\nu}$, where $\alpha\nu > 1$. Combining this with points 1, 3 and 4 above we can apply the results of Lemma~\ref{thm:mmc_unnorm}.
Now, by point 2 above, $X_i, X_j \in \llk \Rightarrow X_i, X_j \in \widehat\llk^{(n)}$ for all large $n$, and so by Lemma~\ref{thm:mmc_unnorm} we have, for all large $n$,
\begin{align*}
d_{ij}^{(n)} = ||\U^{(n)}_{(i),1:c} - \U^{(n)}_{(j),1:c}|| =\mathcal{O}(n^{\frac{9-z}{6}}) = o(n^{-1}),
\end{align*}
by choosing $z$ large enough. On the other hand $X_j\not \in \llk \Rightarrow \exists m\in \N$ s.t. $X_j \not \in \B_{2M\sigma_n^{1-\epsilon}}(\llk) \ \forall n \geq m$, since $\llk$ is closed. This implies that for all large $n$ we have $X_j \not \in \widehat\llk^{(n)}$, and hence, either $X_j\not \in \widehat\Ll$ or $||\U_{(i),1:c} - \U_{(j),1:c}|| \geq \sqrt{\frac{2}{n}} - \mathcal{O}(n^{\frac{15-z}{24}})$. Therefore, again by choosing $z$ large enough, we have
\begin{align*}
nd_{ij}^{(n)} \to \infty,
\end{align*}
as required.
\end{proof}

The case of the normalised Laplacian of the graph where reflexive edges are not removed follows exactly analogously.

\begin{lemma}\label{thm:evec_consistency_2}
Let the conditions of Lemma~\ref{thm:evec_consistency_1} hold.
For each $n \in N$ let $\Ln^{(n)}$ be the normalised Laplacian of the graph with vertices $\widehat \Ll^{(n)}$ and edge weights determined using $K_{\sigma_n}$, and let $\U^{(n)}$ be its eigenvectors and $\D^{(n)}$ the corresponding degree matrix. Now, for each $i,j,n\in\N$, define
 \begin{align*}
 d_{ij}^{(n)} = \left\{\begin{array}{ll}
 ||\V^{(n)}_{(i),1:c}-\V^{(n)}_{(j),1:c}||, \mbox{ if $X_i, X_j \in \widehat\Ll^{(n)}$}; X_i = \widehat\Ll^{(n)}_{(i)}, X_j = \widehat\Ll^{(n)}_{(j)}\\
 1, \mbox{ otherwise,}
 \end{array}\right.
 \end{align*}
where $\V^{(n)} = (\D^{(n)})^{-1/2}\U^{(n)}$. Then, with probability one, for all fixed $i, j \in \N, k\in[c]$ we have,
\begin{align*}
X_i, X_j \in \llk \Rightarrow &\lim_{n\to\infty} nd_{ij}^{(n)} = 0,\\
X_i \in\llk, X_j \not\in \llk \Rightarrow &\lim_{n\to\infty} nd_{ij}^{(n)} = \infty.
\end{align*}
\end{lemma}

\begin{proof}
The proof is exactly analogous to the previous proof.
\end{proof}

When the reflexive edges in the graph are removed, then, as in previous cases, then substantial modifications are needed. These are given explicitly in the following lemma.

\begin{lemma}\label{thm:evec_consistency_3}
Let the conditions of Lemma~\ref{thm:evec_consistency_1} hold, and let $K(x) = \exp(-x^\alpha)$ for some $\alpha > 0$.
For each $n \in N$ let $\Lno^{(n)}$ be the normalised Laplacian of the graph with vertices $\widehat \Ll^{(n)}$ and edge weights determined using $K_{\sigma_n}$, but with reflexive edges removed, and let $\U^{(n)}$ be its eigenvectors and $\D^{(n)}$ the corresponding degree matrix. Now, for each $i,j,n\in\N$, define
 \begin{align*}
 d_{ij}^{(n)} = \left\{\begin{array}{ll}
 ||\V^{(n)}_{(i),1:c}-\V^{(n)}_{(j),1:c}||, \mbox{ if $X_i, X_j \in \widehat\Ll^{(n)}$}; X_i = \widehat\Ll^{(n)}_{(i)}, X_j = \widehat\Ll^{(n)}_{(j)}\\
 1, \mbox{ otherwise,}
 \end{array}\right.
 \end{align*}
where $\V^{(n)} = (\D^{(n)})^{-1/2}\U^{(n)}$. Then, with probability one, for all fixed $i, j \in \N, k\in[c]$ we have,
\begin{align*}
X_i, X_j \in \llk \Rightarrow &\lim_{n\to\infty} n^{1/2}\exp(\sigma_n^{-\alpha\sqrt{\epsilon}})d_{ij}^{(n)} = 0,\\
X_i \in\llk, X_j \not\in \llk \Rightarrow &\lim_{n\to\infty} n^{1/2}\exp(\sigma_n^{-\alpha\sqrt{\epsilon}})d_{ij}^{(n)} = \infty.
\end{align*}
\end{lemma}

\begin{proof}
The proof is similar, but in this case we state the results from Lemma~\ref{thm:consistency} slightly differently. Specifically, we replace point 3 in the proof of Lemma~\ref{thm:evec_consistency_1} with
\begin{enumerate}
\item[3.] With probability one, for all $n$ large enough and for all $k, l \in [c]$, $k\not=l$, we have $d(\widehat \llk^{(n)}, \widehat{ \ell(\lambda, l)}^{(n)})^\alpha - 3M^\alpha\sigma_n^{\alpha(1-\epsilon)} \geq \Delta^\alpha$.
\end{enumerate}
Now, for large $n$, $\sigma_n < \Delta\log(13^8c^9n^{10})^{-1/\alpha}$.
Using the first part of the proof of Lemma~\ref{thm:mmc_norm0}, we thus find that if $X_i, X_j \in \widehat\llk^{(n)}$, then
\begin{align*}
||\V_{(i),1:c}^{(n)} - \V_{(j),1:c}^{(n)}|| \leq n^{3/2}c^{1/2}\exp\left(-\frac{\Delta^\alpha}{2\sigma_n^\alpha}\right)\exp\left(-\frac{M^\alpha\sigma_n^{\alpha(1-\epsilon)}}{2\sigma_n^\alpha}\right)\leq n^{3/2}c^{1/2}\exp\left(-\frac{\Delta^\alpha}{2\sigma_n^\alpha}\right)\\
\Rightarrow n^{1/2}\exp(\sigma_n^{-\alpha\sqrt{\epsilon}})||\V_{(i),1:c}^{(n)} - \V_{(j),1:c}^{(n)}|| \leq n^{2}c^{1/2}\exp\left(-\frac{\Delta^\alpha-2\sigma_n^{\alpha(1-\sqrt{\epsilon})}}{2\sigma_n^\alpha}\right) \to 0 \mbox{ as } n\to \infty,
\end{align*}
similar to the previous proofs.
On the other hand, if $X_i\in \widehat\llk^{(n)}$ and $X_j\in \widehat{\ell(\lambda, l)}^{(n)}$ for $l\not = k$, then again using part of the proof of Lemma~\ref{thm:mmc_norm0}, we have
\begin{align*}
n^{1/2}\exp(\sigma_n^{-\alpha\sqrt{\epsilon}})||\V_{(i),1:c}^{(n)} - \V_{(j),1:c}^{(n)}|| &\geq \exp\left(\frac{2\sigma_n^{\alpha(\epsilon-\sqrt{\epsilon})}-M^\alpha}{2\sigma^{\alpha\epsilon}}\right) \to \infty \mbox{ as } n\to \infty,
\end{align*}
since $\epsilon < 1 \Rightarrow \sqrt{\epsilon} > \epsilon$ and so $\sigma_n^{\alpha(\epsilon-\sqrt{\epsilon})}\to\infty$.
The rest of the proof is analogous to the proof of Lemma~\ref{thm:evec_consistency_1}.
\end{proof}

The above three results show that the level set components are consistently estimated by spectral clustering applied to the estimated level set, $\widehat\Ll^{(n)}$, assuming that the number of components is known. Dependence on this value arises in the fact that the terms $d_{ij}^{(n)}$ are computed from the first $c$ columns of the eigenvectors of $\L^{(n)}, \Ln^{(n)}$ and $\Lno^{(n)}$. The final three results show that $c$ can be consistently estimated by considering scaled sequences of the eigenvalues of the various Laplacian matrices. Combining these with the previous results therefore ensures that the level set components can be consistently estimated using the approach described herein. Note that the resuirements on the sequence of scaling parameters, $\{\sigma_n\}_{n=1}^\infty$ occur in their strictest form in the following three results. Using a sequence satisfying the conditions which follow therefore ensures all results of this section hold almost surely.

\begin{lemma}\label{thm:eval_consistency_1}
Let the conditions of Lemma~\ref{thm:evec_consistency_1} hold, except assume now that $\sigma_n < \log(n)^{-\nu}$ for all large $n$, where in this case $\nu > \alpha^{-1}\epsilon^{-1/2}$.
%
%
%
For each $l\in\left[\big|\widehat \Ll^{(n)}\big|\right]$, let $e_l^{(n)}$ be the $l$-th eigenvalue of $\L^{(n)}$.
Then we have,
\begin{align*}
\frac{e^{(n)}_l}{f_n} &\xrightarrow{a.s.} 0,  \mbox{ for } l\in [c], \ \ 
\frac{e^{(n)}_{c+1}}{f_n} \xrightarrow{a.s.} \infty,
\end{align*}
%
%
for any sequence $\{f_n\}_{n=1}^\infty$ of the form $f_n = Hn^hK(\sigma_n^{-\sqrt{\epsilon}})$, where $H>0$ and $h\in \R$ are any fixed constants.
\end{lemma}

\begin{proof}
We again use the beginning of the proof of Lemma~\ref{thm:evec_consistency_1} to obtain points 1--4 from the results of Lemma~\ref{thm:consistency}.
By Lemma~\ref{thm:eval_bounds_1} we thus have,
\begin{align*}
\sum_{l=1}^c e^{(n)}_l &\leq ncK\left(\frac{\Delta}{\sigma_n}\right)\\
e^{(n)}_{c+1} &\geq \frac{1}{9n^2(c+1)^4}K\left(\frac{M}{\sigma_n^\epsilon}\right)-n^3cK\left(\frac{\Delta}{\sigma_n}\right).
\end{align*}
For any $H>0$ and $h\in \R$ and any $l\in [c]$, we thus have, for $n$ large enough that $\sigma_n^{1-\sqrt{\epsilon}} < \frac{1}{2}\Delta$,
\begin{align*}
\frac{e^{(n)}_l}{Hn^hK(\sigma_n^{-\sqrt{\epsilon}})} &\leq \frac{cK(\Delta/\sigma_n)}{Hn^{h-1}K(\sigma_n^{-\sqrt{\epsilon}})} \leq \frac{cA}{Hn^{h-1}}\exp\left(-\left(\frac{\Delta - \sigma_n^{1-\sqrt{\epsilon}}}{\sigma_n}\right)^\alpha\right)\\
&\leq \frac{cA}{Hn^{h-1}B}\exp\left(-\frac{\Delta^\alpha}{2^\alpha\sigma_n^{\alpha}}\right)
\leq \frac{cA}{Hn^{h-1}B}\exp\left(-\frac{\Delta^\alpha}{2^\alpha}\log(n)^{\alpha\nu}\right)\to 0 \mbox{ as } n\to \infty,
\end{align*}
since $\alpha\nu > \alpha\sqrt{\epsilon}\nu > 1$.
In addition, we have
\begin{align*}
\frac{e^{(n)}_{c+1}}{Hn^hK(\sigma_n^{-\sqrt{\epsilon}})} &\geq \frac{K(M/\sigma_n^\epsilon)}{9Hn^{h+2}(c+1)^4K(\sigma_n^{-\sqrt{\epsilon}})} - \frac{cK(\Delta/\sigma_n)}{Hn^{h-3}K(\sigma_n^{-\sqrt{\epsilon}})}
\end{align*}
Now, since $\epsilon < 1$, we have $\sqrt{\epsilon}>\epsilon$, and so for $n$ large enough, we have $M\sigma_n^{\sqrt{\epsilon}-\epsilon} < \frac{1}{2}$. Therefore, using assumption AK3 and the fact that $K$ is strictly positive,
\begin{align*}
\frac{K(M/\sigma_n^\epsilon)}{9Hn^{h+2}(c+1)^4K(\sigma_n^{-\sqrt{\epsilon}})} &\geq 
\frac{1}{9AHn^{h+2}(c+1)^4}\exp\left(\left(\frac{1 - M\sigma_n^{\sqrt{\epsilon}-\epsilon}}{\sigma_n^{\sqrt{\epsilon}}}\right)^\alpha\right)\\
 &\geq \frac{1}{9AHn^{h+2}(c+1)^4}\exp\left(\frac{1}{2^\alpha\sigma_n^{\alpha\sqrt{\epsilon}}}\right)\\
&\geq \frac{1}{9AHn^{h+2}(c+1)^4}\exp\left(\frac{1}{2^\alpha}\log(n)^{\alpha\nu\sqrt{\epsilon}}\right)\to\infty \mbox{ as } n\to\infty,\\
\mbox{ and }\frac{cK(\Delta/\sigma_n)}{Hn^{h-3}K(\sigma_n^{-\sqrt{\epsilon}})} &\leq \frac{cA}{Hn^{h-3}B}\exp\left(-\frac{\Delta^\alpha}{2^\alpha}\log(n)^{\alpha\nu}\right)\to 0 \mbox{ as } n\to \infty,
\end{align*}
as before. This proves the result.
%
%
\end{proof}

\begin{lemma}\label{thm:eval_consistency_2}
Let the conditions of Lemma~\ref{thm:eval_consistency_1} hold. 
 Let $\Ln^{(n)}$ be the normalised Laplacian of the graph with vertices $\widehat \Ll^{(n)}$ and edge weights determined using $K_{\sigma_n}$. For each $l\in\left[\big|\widehat \Ll^{(n)}\big|\right]$, let $e_l^{(n)}$ be the $l$-th eigenvalue of $\Ln^{(n)}$.
Then we have,
\begin{align*}
\frac{e^{(n)}_l}{f_n} &\xrightarrow{a.s.} 0,  \mbox{ for } l\in [c], \ \ 
\frac{e^{(n)}_{c+1}}{f_n} \xrightarrow{a.s.} \infty,
\end{align*}
for any sequence $\{f_n\}_{n=1}^\infty$ of the form $f_n = Hn^hK(\sigma_n^{-\sqrt{\epsilon}})$, where $H>0$ and $h\in \R$ are any fixed constants.
\end{lemma}

\begin{proof}
The proof is exactly analogous to the previous proof.
\end{proof}

\begin{lemma}\label{thm:eval_consistency_3}
Let the conditions of Lemma~\ref{thm:eval_consistency_1} hold, and let $K(x) = \exp(-x^\alpha)$.
 Let $\Lno^{(n)}$ be the normalised Laplacian of the graph with vertices $\widehat \Ll^{(n)}$ and edge weights determined using $K_{\sigma_n}$, but with reflexive edges removed. For each $l\in\left[\big|\widehat \Ll^{(n)}\big|\right]$, let $e_l^{(n)}$ be the $l$-th eigenvalue of $\Lno^{(n)}$.
Then we have,
\begin{align*}
\frac{e^{(n)}_l}{f_n} &\xrightarrow{a.s.} 0,  \mbox{ for } l\in [c], \ \ 
\frac{e^{(n)}_{c+1}}{f_n} \xrightarrow{a.s.} \infty,
\end{align*}
for any sequence $\{f_n\}_{n=1}^\infty$ of the form $f_n = Hn^hK(\sigma_n^{-\sqrt{\epsilon}})$, where $H>0$ and $h\in \R$ are any fixed constants.
\end{lemma}

\begin{proof}
Using the same approach as in the proof of Lemma~\ref{thm:eval_consistency_1} we have, now using Lemma~\ref{thm:eval_bounds_3}, that
\begin{align*}
\sum_{l=1}^c e^{(n)}_l &\leq nc\frac{K\left(\Delta/\sigma_n\right)}{K(M/\sigma_n^{\epsilon})}\\
e^{(n)}_{c+1} &\geq \frac{K\left(M/\sigma_n^\epsilon\right)}{9n^3(c+1)^4}-n^3c\frac{K\left(\Delta/\sigma_n\right)}{K(M/\sigma_n^{\epsilon})}.
\end{align*}
The first term in $e_{c+1}^{(n)}$, divided by $f_n$ tends to $\infty$ almost surely, almost exactly as before. The second term in $e^{(n)}_{c+1}$ and in the first $c$ eigenvalues converge to zero fast enough, almost surely, since for $n$ large enough that $M^\alpha\sigma_n^{\alpha(1-\epsilon)} < \frac{1}{4}\Delta^\alpha$ and $\sigma_n^{\alpha(1-\sqrt{\epsilon})} < \frac{1}{4}\Delta^\alpha$, we have
\begin{align*}
\frac{cn^3K(\Delta/\sigma_n)}{Hn^hK(\sigma_n^{-\sqrt{\epsilon}})K(M/\sigma_n^\epsilon)} &= \frac{c}{Hn^{h-1}}\exp\left(-\frac{\Delta^\alpha}{\sigma_n^\alpha} + \frac{M^\alpha}{\sigma_n^{\alpha\epsilon}}+\frac{1}{\sigma_n^{\alpha\sqrt{\epsilon}}}\right)\\
& \leq \frac{c}{Hn^{h-1}}\exp\left(-\frac{\Delta^\alpha}{2\sigma_n^\alpha}\right) \to 0 \mbox{ as } n\to \infty.
\end{align*}
The rest of the proof follows as before.
\end{proof}

\section{Discussion}\label{sec:conclusion}

In this paper we investigated the relationships between spectral clustering and the problems of maximum margin clustering and estimation of level sets of a probability density. Although these two problems are not usually associated with one another, by applying a maximum margin clustering method to a truncated sample whose low-density points have been removed, it is intuitively the case that such an approach is likely to recover an approximation of the components of a level set of the underlying density. We extended existing theory on the connection between spectral clustering and density level sets by considering multiple versions of spectral clustering, by considering a broader class of kernels including the ubiquitous Gaussian kernel, and importantly achieve consistent estimation with a sequence of scaling parameters which decreases with the sample size. Existing convergence results for spectral clustering assume a fixed bandwidth kernel is used. Although intuitive, as far as we are aware the connection between spectral clustering and maximum margin clustering in the general case has not been made explicit until now.


\end{document}